\documentclass[10pt,letterpaper]{article}
\usepackage{naaclhlt2016}
\usepackage{stfloats}

\usepackage{natbib}
\usepackage[utf8]{inputenc} 
\usepackage[T1]{fontenc}    
\usepackage{hyperref}       
\usepackage{url}            
\usepackage{booktabs}       
\usepackage{amsfonts}       
\usepackage{nicefrac}       
\usepackage{microtype}      
\usepackage{xcolor} 
\usepackage{amsmath}

\usepackage{amssymb}
\usepackage{mathtools}
\usepackage{amsthm}
\usepackage{graphicx}
\usepackage{subcaption}
\usepackage{multirow}
\usepackage{float}

\usepackage{microtype}
\usepackage{graphicx}
\usepackage{subcaption}
\usepackage{booktabs} 
\usepackage{microtype}
\usepackage{setspace}
\usepackage{hyperref}
\usepackage{xcolor}
\usepackage{soul}
\usepackage{tabularx}

\theoremstyle{plain}
\newtheorem{theorem}{Theorem}[section]

\theoremstyle{definition}

\newtheorem{assumption}[theorem]{Assumption}
\theoremstyle{remark}

\usepackage{algorithm}
\usepackage{algorithmic}
\usepackage[normalem]{ulem}

\newcommand{\squishlist}{
   \begin{list}{$\bullet$}
    { \setlength{\itemsep}{-1pt}      \setlength{\parsep}{3pt}
      \setlength{\topsep}{1pt}       \setlength{\partopsep}{0pt}
      \setlength{\leftmargin}{1.5em} \setlength{\labelwidth}{1em}
      \setlength{\labelsep}{0.5em} } }
\newcommand{\squishend}{
    \end{list}  }
    
\naaclfinalcopy 
\def\naaclpaperid{***} 

\title{TriOpt: A Scalable Algorithm for Linear Causal Discovery} 

\author{Rafat Ashraf Joy\\
	    University of Illinois Chicago\\
	    {\tt rjoy4@uic.edu}
	  \And
	Elena Zheleva\\
  	University of Illinois Chicago\\
	    {\tt ezheleva@uic.edu}}

\date{}

\begin{document}

\maketitle

\begin{abstract}
Learning causal relations from observational data is challenging because the graph search space grows super-exponentially with the number of variables. Ordering-based methods reduce this space by first identifying the topological ordering, whereas continuous optimization methods explore most likely regions of the space by casting DAG learning as a differentiable objective with an acyclicity constraint. Despite their conceptual appeal, both paradigms face significant scalability limitations in high-dimensional settings, restricting their practical applicability. In this work, we introduce a new formulation for linear causal discovery that tightly integrates these two paradigms to achieve substantial gains in scalability without sacrificing accuracy. Our approach, TriOpt, decomposes the problem into two efficient stages. First, it recovers the topological ordering by exploiting the Sherman–Morrison rank-1 downdate together with the additive structure of linear kernels, enabling fast and scalable ordering estimation. Second, given this ordering, we reformulate structure learning as a convex continuous optimization problem that entirely avoids the need for enforcing costly acyclicity constraints. We theoretically show that, under the true ordering, TriOpt exactly recovers the underlying linear DAG. Empirically, across synthetic, semi-synthetic, and real-world datasets, TriOpt achieves orders-of-magnitude speedups over state-of-the-art linear causal discovery methods in high-dimensional regimes, while maintaining comparable or superior accuracy. 
\end{abstract}

\section{Introduction}
Causal discovery, the process of learning graphs that capture cause-and-effect relationships and explain the underlying causal mechanism of data generation, is a fundamental problem in machine learning. It has wide range of applications in molecular network analysis \citep{kelly2022review}, epidemiology \citep{10.1093/ije/dyac135} and recommender systems \citep{10.1145/3680296}. While the gold standard of learning cause-and-effect relationships is performing randomized controlled trials, such interventional experiments are sometimes infeasible due to cost, ethical concerns, or mere impossibility. Thus, many research studies have focused on learning these relationships from observational data.

Ordering-based and continuous-optimization methods have recently emerged as important paradigms for causal discovery, often outperforming traditional approaches. Ordering-based methods such as SCORE \citep{pmlr-v162-rolland22a}, CaPS \citep{xu2024ordering} and DiffAN \citep{sanchez2023diffusion} first obtain a topological ancestral order via score matching, then use postprocessing methods such as CAM \citep{buhlmann2014cam} to prune the spurious edges. Some of these methods utilize kernelized Stein discrepancy based score matching for ordering. In this approach, the kernel matrix and its inverse must be recomputed each time a variable is identified as a leaf, leading to total computational costs of $\mathcal{O}(d^2n^2)$ and $\mathcal{O}(dn^3)$, respectively. The biggest disadvantage of ordering-based methods is that the postprocessing step can be prohibitively slow in high-dimensional settings and it is much more computationally expensive than the ordering step. The computational complexity of the CAM pruning is 
$\mathcal{O}\!\left(d \cdot r^{(n,d)}\right)$, where $r^{(n,d)}$ denotes the complexity of training a generalized additive model.

On the other hand, continuous optimization based methods like NOTEARS \citep{zheng2018dags}, GOLEM \citep{ng2020role} and DAGMA \citep{bello2022dagma} formulate the problem as a continuous optimization task. One key advantage of these methods is superior optimization performance as well as reduced computational cost when leveraging GPUs. However, on the downside, these methods need to enforce an acyclicity constraint with time complexity $\mathcal{O}(d^3)$, which requires costly matrix exponentials and is often susceptible to vanishing gradients.  


\textbf{Contributions}. In this work, we introduce $TriOpt$ (Three Fold Optimization), a framework which combines the strengths of ordering-based and continuous optimization causal discovery while mitigating their weaknesses to make causal discovery drastically more scalable for higher-dimensional settings. $TriOpt$ first efficiently finds the topological ordering by score matching via linear kernel, then instead of pruning, it solves a novel continuous optimization problem. 

We focus on linear causal discovery, DAG learning under linear SEM assumptions, which is widely used in practice and can work quite well in many real-world settings~\citep{ng2020role,zheng2018dags}. Even when the underlying system has nonlinear relationships, linear causal discovery can sometimes give good approximations with lower complexity and better interpretability than nonlinear causal discovery methods~\citep{fujiwara2023causal}.

To accelerate causal discovery for linear models, 
we drastically reduce the computational overhead of causal ordering without sacrificing accuracy. Previous Stein ordering methods incur immense runtime bottlenecks by naively recomputing and re-inverting the kernel matrix upon every feature removal. $TriOpt$ circumvents this computationally expensive process by exploiting the additivity property of linear kernels and leveraging Sherman Morrison rank-1 downdate\citep{Sherman1950AdjustmentOA}.
    In the ordering phase, $TriOpt$ brings down the cost of kernel matrix and its inverse computation from $\mathcal{O}(d^2n^2)$ and $\mathcal{O}(dn^3)$ to $\mathcal{O}(dkn^2)$ and $\mathcal{O}(dn^2 + kn^3)$ respectively - where $k$ is much smaller than $d$. Our empirical results show that these two contributions make the ordering phase more accurate and much more time efficient. 
    
Moreover, we establish a novel theoretical result that, when the ground-truth causal ordering is known, then optimizing the $TriOpt$ objective under a hard upper-triangular constraint (enabled by the ordering step) provably recovers the true DAG in the population setting. 
Using this, we formulate the post-ordering step as a novel optimization problem inspired by NOTEARS \citep{zheng2018dags} but without the computationally expensive acyclicity constraint. A second, equally important point is that dropping the acyclicity constraint also renders the optimization problem convex - which guarantees a global minimum. Although $TriOpt$ optimization has the same $\mathcal{O}(d^3)$ complexity class as NOTEARS and other continuous optimization methods, empirically it is much faster as its cost comes from pure matrix multiplications rather than matrix exponentials. 
    
With these optimizations, our proposed $TriOpt$ framework achieves substantial runtime improvements over both state-of-the-art ordering-based and continuous optimization approaches, while attaining better or comparable accuracy in high-dimensional linear causal discovery tasks. 





\section{Related Work}

\textbf{Ordering-Based Causal Discovery}. Ordering-based methods estimate the causal graph in two steps.
First, they find a topological order of the nodes. In this order, a node can only be a parent of nodes that follow.
Then they build the graph using this order and pruning spurious edges. 
Assuming Gaussian noise and that each function is nonlinear and twice continuously differentiable, SCORE \citep{pmlr-v162-rolland22a} introduces an
identifiability criterion based on the variance of diagonal Hessian entries. It picks as the leaf the remaining variable whose score variance is the least, since a true leaf is least influenced by other variables. CaPS \citep{xu2024ordering}  identifies a leaf variable as the one whose expected diagonal Hessian of the log-density score is largest, which under the ANM assumption corresponds to having no children. Once the order is obtained, both SCORE and CaPS employes CAM pruning~\citep{buhlmann2014cam} to get the final DAG. Unlike SCORE's leaf identification criterion which holds only for nonlinear additive models, CaPS leaf identification criterion is applicable in both linear and nonlinear relations. 
Both SCORE and CaPS utilize an RBF kernel via kernelized Stein discrepancy for score matching. However, a common drawback of these methods is the need to recompute the kernel matrix and kernel inverse every time a leaf node is removed in Stein ordering, which we address in our work. 

\textbf{Continuous Optimization Causal Discovery}. NOTEARS \citep{zheng2018dags} is the first to formulate DAG learning as a continuous optimization problem. The objective function consists of a least squares loss with $l1$ penalty and a hard DAG constraint. The objective is optimized using
the Augmented Lagrangian method with L-BFGS.  
GOLEM \citep{ng2020role} shows that replacing NOTEARS' least-squares objective with a likelihood-based objective allows using only soft sparsity and DAG constraints, under mild assumptions. 
This eliminates the need for a hard DAG constraint and makes the optimization problem easier to solve. DAGMA \citep{bello2022dagma} introduced a new acyclicity constraint, which is based on the log-determinant of the weighted adjacency matrix. This new acyclicity constraint enables detecting larger cycles and better behaved gradients. 
However, the constraint also introduces a significant computational overhead, which impacts the efficiency and scalability of continuous optimization-based techniques.
\section{Preliminaries}



\subsection{Problem Definition: Linear Causal Discovery}
We aim to recover the causal DAG \(G=(V,E)\) over \(d\) variables \(x=(x_1,\dots,x_d)\) from \(n\) i.i.d.\ observational samples \(X\in\mathbb{R}^{n\times d}\). The DAG defines an SCM \(x_i := f_i(\mathrm{Pa}_G(x_i))+\varepsilon_i\), where the noise variables \(\{\varepsilon_i\}_{i=1}^d\) are mutually independent. In this work, we consider linear DAG models, where \(x_i=\sum_{j\in \mathrm{Pa}_G(i)} B_{j,i}x_j+\varepsilon_i\), equivalently \(X=B^\top X+\varepsilon\). Our goal is to estimate the weighted adjacency matrix \(B\in\mathbb{R}^{d\times d}\), which encodes the causal graph \(G\).

\subsection{Score Matching}
Score (or Stein's score) is defined as the gradient of the log probability with respect to the data itself $S(x) = \nabla_{x}logp_d(x)$. Score matching \citep{JMLR:v6:hyvarinen05a} aims to estimate the score function $\nabla_x \log p_d(x)$ of an unknown probability
density $p_d(x)$ only using samples $x \sim p_d$, without requiring an explicit parametric form of the
density. Let $s_\theta(x)$ denote a parametric score model. The score matching objective is:
\begin{equation}
\mathcal{L}_{\mathrm{SM}}(s_\theta)
= \frac{1}{2}\,\mathbb{E}_{p_d}\!\left[\left\| s_\theta(X) - \nabla_x \log p_d(X) \right\|^2 \right].
\label{eq:sm_original}
\end{equation}

Since the true score $\nabla_x \log p_d(x)$ is unknown, this objective cannot be
evaluated directly. However, by applying integration by parts and assuming boundary
conditions, the loss can be rewritten such that it does not depend on
$\nabla_x \log p_d(x)$ \citep{JMLR:v6:hyvarinen05a}:
\begin{equation}
\mathcal{L}_{\mathrm{SM}}(s_\theta)
= \frac{1}{2}\,\mathbb{E}_{p_d}\!\left[\| s_\theta(X) \|^2 \right]
+ \mathbb{E}_{p_d}\!\left[ \nabla_x \cdot s_\theta(X) \right].
\label{eq:sm_rewritten}
\end{equation}
SCORE \citep{pmlr-v162-rolland22a} adapts the method developed in~\citep{li2018gradient} to estimate the score at the sample points which is approximating 

$\textbf{G} \equiv (\nabla \log p(\mathbf{x}^1), \ldots, \nabla \log p(\mathbf{x}^n))^T \in \mathbb{R}^{n\times d}$.

This estimator relies on Stein's identity~\citep{stein1972bound}. For any test function $\mathbf{h}:\mathbb{R}^d \rightarrow \mathbb{R}^{d'}$ satisfying boundary conditions, we have:
\begin{equation}\label{eq:stein_identity}
    \mathbb{E}_p[\mathbf{h}(\mathbf{x}) \nabla \log p(\mathbf{x})^T + \nabla \mathbf{h}(\mathbf{x})] = \mathbf{0}.
\end{equation}

\subsection{Topological Ordering}
Obtaining the topological ordering is the first step in ordering based causal discovery methods. Given a causal directed acyclic graph (DAG) $\mathcal{G}$, a causal order
$\pi : \mathcal{V} \to \mathcal{V}$ is a permutation of the nodes
$\mathcal{V}=\{1,\ldots,D\}$ such that $\pi(i) < \pi(j)$ for every directed edge
$i \to j$ in $\mathcal{G}$. 

CaPS~\citep{xu2024ordering} introduces a discriminant criterion that applies to both linear and nonlinear causal relations. They establish sufficient conditions for identifiability that do not rely on functional nonlinearity.

\begin{assumption}
\label{assumption-sufficient_conditions}
\textbf{(Sufficient conditions for identifiability).}
The topological ordering of a causal graph is identifiable if at least one of the following conditions holds:
\begin{itemize}
    \item[(i)] \textbf{Non-decreasing noise variance.} For any two nodes $i$ and $j$,
    \[
        \sigma_j \ge \sigma_i \quad \text{whenever} \quad \pi(i) < \pi(j).
    \]
    
    \item[(ii)] \textbf{Non-weak causal effect.} For any non-leaf node $x_j$,
    \[
        \sum_{i \in \mathrm{Ch}(j)} \frac{1}{\sigma_i^2}
        \mathbb{E}\!\left[
            \left(
            \frac{\partial f_i}{\partial x_j}\bigl(pa_i(x)\bigr)
            \right)^2
        \right]
        \;\ge\;
        \frac{1}{\sigma_{\min}^2} - \frac{1}{\sigma_j^2}.
    \]
\end{itemize}
\end{assumption}

Here, $\sigma_{\min}$ denotes the minimum noise variance across all nodes. CaPS~\citep{xu2024ordering} introduces the following theorem. Let
\[
j \;=\; \arg\max \Bigl(
\operatorname{diag}\Bigl(
\mathbb{E}\Bigl[\tfrac{\partial s(x)}{\partial x}\Bigr]
\Bigr)
\Bigr).
\]
Then \(x_j\) is a leaf node. Here, \(s(x)\) denotes the score function.

\section{Linear Causal Discovery with TriOpt}

To learn the causal graph under linear assumptions, $TriOpt$ first identifies the ancestral ordering via score matching with a linear kernel. Then, it determines specific cause-effect relationships through a convex optimization objective. In the following section, we detail the algorithm and present its theoretical guarantees. Algorithms \ref{alg:stein_ordering} and \ref{alg:triopt} provide the pseudo-code for the ordering and optimization steps, respectively. A detailed discussion of their computational complexity is in Appendix \ref{computational-complexity}.

\subsection{Step 1: Ancestral Ordering} 
In this work, we employ a linear kernel for score matching which allows us to take advantage of its additivity property to avoid recomputing the kernel matrix at every step of the Stein ordering. In contrast, CaPS~\citep{xu2024ordering} requires that the RBF kernel matrix and its inverse are recomputed every time a feature is removed in Stein ordering, making Stein ordering very expensive. For the linear kernel, the test functions are identity map $\mathbf{h}(\mathbf{x}) = \mathbf{x}$ (assuming centered data). Under the Monte Carlo approximation of Eq.~\eqref{eq:stein_identity}, the Stein gradient estimator is given by the regularized solution:
\begin{align}
    \hat{\textbf{G}}^{\text{Stein}} &= -(\textbf{K} + \eta \textbf{I})^{-1} \langle \nabla, \textbf{K}\rangle, \label{eq:stein_estimator}
\end{align}
where $\textbf{K} \in \mathbb{R}^{n \times n}$ is the kernel matrix and $\langle \nabla, \textbf{K}\rangle \in \mathbb{R}^{n \times d}$ is the kernel gradient matrix.

For the linear kernel, the kernel matrix is the Gram matrix of the data $\textbf{K} = \mathbf{X}\mathbf{X}^\top$. The gradient term becomes:
\vspace{-5pt}
\[
(\langle \nabla, \mathbf{K} \rangle)_{ij}
= \sum_{k=1}^n \nabla_{\mathbf{x}^k_j} \kappa(\mathbf{x}^i, \mathbf{x}^k)
= \sum_{k=1}^n \mathbf{x}^i_j
= n\,\mathbf{x}^i_j .
\]

In matrix notation, this is $\langle \nabla, \textbf{K}\rangle = n\mathbf{X}$. Consequently, the linear closed-form estimator is:

\begin{equation}\label{eq:linear_stein}
    \hat{\textbf{G}}^{\text{Stein}} = -(\mathbf{X}\mathbf{X}^\top + \eta \textbf{I})^{-1} (n\mathbf{X}).
\end{equation}

To estimate the Hessian diagonal $\mathbf{H}$ from the Stein score estimator $\hat{\mathbf{G}}^{\text{Stein}}$, we leverage the second-order Stein identity. For a density $p(\mathbf{x})$, the expected Hessian of the log-likelihood relates to the score function $\mathbf{s}(\mathbf{x}) = \nabla \log p(\mathbf{x})$ via:
\begin{equation}
    \mathbb{E}_{p(\mathbf{x})} \left[ \nabla^2 \log p(\mathbf{x}) \right] = -\mathbb{E}_{p(\mathbf{x})} \left[ \mathbf{s}(\mathbf{x}) \mathbf{s}(\mathbf{x})^\top \right].
\end{equation}

The resulting Hessian diagonal matrix $\mathbf{H} \in \mathbb{R}^{n \times d}$ is computed as:
\begin{equation}
    \mathbf{H} = -(\hat{\mathbf{G}}^{\text{Stein}} \odot \hat{\mathbf{G}}^{\text{Stein}})
\end{equation}
where $\odot$ denotes the Hadamard product. The causal order is then determined by iteratively identifying the leaf node $j^*$ that maximizes the empirical mean of the Hessian proxy~\citep{xu2024ordering}:
\begin{equation}
    j^* = \arg\max_{j \in \mathcal{S}} \frac{1}{n} \sum_{i=1}^n H_{ij}.
\end{equation}

\textbf{Inverse Kernel via Sherman Morrison Downdate}.
During the ordering phase, each time a leaf node is eliminated, the kernel inverse $(\textbf{K} + \eta \textbf{I})^{-1}$ has to be recomputed. However, this incurs $\mathcal{O}(n^3)$ time complexity each time, where n is the number of samples. 

Sherman Morrison rank-1 downdate \citep{Sherman1950AdjustmentOA,elden2010updating} allows us to recompute the kernel inverse efficiently with $\mathcal{O}(n^2)$ time complexity. It exploits the fact that when the kernel is linear, eliminating a leaf node corresponds to a low-rank
modification. 

Let $
\mathbf{A} = \mathbf{K} + \eta \mathbf{I},
$. 
The removal of a leaf node can be expressed as a rank-1 downdate
$\mathbf{A}' = \mathbf{A} + \mathbf{u}\mathbf{v}^\top$ 
where \(\mathbf{u}, \mathbf{v} \in \mathbb{R}^n\).
By the Sherman--Morrison formula, the inverse of \(\mathbf{A}'\) is given by
\[
(\mathbf{A} + \mathbf{u}\mathbf{v}^\top)^{-1}
=
\mathbf{A}^{-1}
-
\frac{\mathbf{A}^{-1}\mathbf{u}\mathbf{v}^\top\mathbf{A}^{-1}}
{1 + \mathbf{v}^\top\mathbf{A}^{-1}\mathbf{u}},
\]

 One downside of this method is the denominator \(1 + \mathbf{v}^\top\mathbf{A}^{-1}\mathbf{u}\) can approach zero, which would make the estimate unstable.   
 
\begin{algorithm}[t] 
\small 
\caption{Fast Stein Ordering with Linear Kernel}
\label{alg:stein_ordering}
\begin{algorithmic}[1]
\setstretch{0.9} 
\renewcommand{\algorithmicrequire}{\textbf{Input:}}
\renewcommand{\algorithmicensure}{\textbf{Output:}}

\REQUIRE $X \in \mathbb{R}^{n \times d}$, $\eta_G$; \textbf{Output:} Causal Order $\pi$

\STATE Center $X$; $\pi \leftarrow [\,]$, $\mathcal{S} \leftarrow \{0, \dots, d-1\}$
\STATE $K \leftarrow XX^\top + \eta_G I_n$
\STATE $L \leftarrow \text{Cholesky}(K); \quad K^{-1} \leftarrow (L^\top)^{-1} L^{-1}$

\FOR{$k = 1$ \textbf{to} $d-1$}
    \STATE \textbf{Step 1: Hessian.} $\nabla K \leftarrow nX$; $G \leftarrow -K^{-1} \nabla K$; $H \leftarrow - (G \odot G) $
    \STATE \textbf{Step 2: Selection.} $s \leftarrow \mathbb{E}_{\text{rows}}[H]$; $l \leftarrow \arg\max s$; $\pi \leftarrow \pi \cup \{\mathcal{S}[l]\}$
    
    \STATE \textbf{Step 3: SM Downdate.} Let $u$ be col $l$ of $X$; $v \leftarrow K^{-1}u$; $\alpha \leftarrow 1 - u^\top v$
    \IF{$|\alpha| < 10^{-7}$ or periodic re-inversion}
        \STATE Recompute $K \leftarrow XX^\top + \eta_G I_n$ and $K^{-1}$
    \ELSE
        \STATE $K^{-1} \leftarrow K^{-1} + \frac{vv^\top}{\alpha}$
    \ENDIF
    \STATE Remove col $l$ from $X$ and element $l$ from $\mathcal{S}$
\ENDFOR
\STATE $\pi \leftarrow \text{reverse}(\pi \cup \{\mathcal{S}[0]\})$; \textbf{return} $\pi$
\end{algorithmic}
\end{algorithm}

\subsection{Step 2: Contiunuous Optimization}
Under a fixed ordering, 
the adjacency matrix is guaranteed to be strictly upper triangular. This allows us to 
define the causal graph search as a convex optimization problem without an acyclicity constraint: 
\vspace{-8pt}
\[
F(W) = \ell(W; \mathbf{X}) + \lambda \lVert W \rVert_{1}
      = \frac{1}{2n} \lVert \mathbf{X} - \mathbf{X}W \rVert_{F}^{2}
        + \lambda \lVert W \rVert_{1}
\]
subject to $W_{ij} = 0$ for all $i \ge j$

We implemented a Gram-matrix–optimized variant of this optimization while enforcing a hard upper triangular constraint: 
\[
W \leftarrow W \odot M,
\]
where $M = \mathrm{triu}(\mathbf{1}, 1)$ is a binary mask.
We precompute the  Gram matrix $G = \frac{1}{n} X^\top X$ and its trace $\mathrm{Tr}(G)$. So $\ell(W; \mathbf{X})$ can be expressed as follows: 
\[
\begin{aligned}
\ell(W; \mathbf{X})
&= \frac{1}{2n} \| X - XW \|_F^2 \\
&= \frac{1}{2} \left[ \mathrm{Tr}(W^\top G W) - 2\,\mathrm{Tr}(G W) + \mathrm{Tr}(G) \right].
\end{aligned}
\]

To encourage sparsity, we apply a smooth $\ell_1$ (Charbonnier) regularization:
\[
\|W\|_{1,\epsilon}
= \sum_{i,j} \sqrt{W_{ij}^2 + \epsilon^2},
\]
where $\epsilon > 0$ ensures differentiability. 
The final optimization objective is as follows:
\[
\min_W \;
\frac{1}{2} \mathrm{Tr}(W^\top G W)
- \mathrm{Tr}(G W)
+ \frac{1}{2} \mathrm{Tr}(G)
+ \lambda \|W\|_{1,\epsilon} 
\]
subject to $W_{ij} = 0$ for all $i \ge j$.

\begin{algorithm}[htbp]
\caption{Upper Triangular Optimization}
\label{alg:triopt}
\begin{algorithmic}[1]
\REQUIRE Data $X \in \mathbb{R}^{n\times d}$, Ordering $\pi$.
\ENSURE Adjacency matrix $\widehat{B} \in \mathbb{R}^{d\times d}$. 

\STATE Reorder data: $X_{\pi} \leftarrow X P_\pi$
 
\STATE Solve optimization with upper triangular constraint:
\\

 $\widehat{W}_{\pi} \leftarrow 
\arg\min_{W}
\;
\frac{1}{2n}
\left\| X_{\pi} - X_{\pi} W \right\|_F^2 +  \lambda \lVert W \rVert_{1,\epsilon}$

\STATE Recover original variable order:
\STATE $\widehat{B} \leftarrow P_\pi \widehat{W}_{\pi} P_\pi^\top$

\STATE \textbf{return} $\widehat{B}$
\end{algorithmic}
\end{algorithm}

\subsection{Identifiability} 
The correctness of the ordering step relies on the identifiability criteria (non-decreasing noise variance or non-weak causal effect) established for score-matching methods like CaPS, which we detailed in Assumption \ref{assumption-sufficient_conditions}. 
For the optimization step, we introduce a theorem according to which optimizing the TriOpt objective under a hard upper-triangular constraint is guaranteed to recover the true DAG if the ground-truth causal ordering is known.  
\begin{theorem}[Exact recovery of a linear DAG under known ordering]
\label{thm:notears_known_order_improved}
Suppose, $X = (X_1, \dots, X_d)^T \in \mathbb{R}^d$ is generated by the linear structural equation model:
\[
X = B^* X + Z,
\]
the following assumptions hold:
\begin{enumerate}
    \item  $B^* \in \mathbb{R}^{d \times d}$ is strictly lower triangular (consistent with the known topological ordering).
    \item  The noise components $Z = (Z_1, \dots, Z_d)^T$ are mutually uncorrelated with zero mean. That is, $\mathbb{E}[Z] = 0$ and the covariance matrix is diagonal: $\mathrm{Cov}(Z) = \Omega = \mathrm{diag}(\omega_1^2, \dots, \omega_d^2)$.
    \item  The noise variances satisfy $\omega_j^2 > 0$ for all $j \in \{1, \dots, d\}$.
\end{enumerate}

Then the least squares objective for $W \in \mathbb{R}^{d \times d}$:
\[
\min_{W \in \mathbb{R}^{d \times d}} \;\; \frac{1}{2} \mathbb{E} \|X - W^TX\|_2^2
\quad \text{subject to} \quad W_{ij} = 0 \;\; \forall \, i \geq j.
\]
has a unique global minimizer $W^* = (B^*)^T$. 
\end{theorem}

 We prove the theorem by showing that the invertibility of $(I - B^*)$ implies a positive definite covariance $\Sigma_X$, which guarantees that the objective function decomposes into $d$ strictly convex least-squares sub-problems. Because the inverse matrix $(I - B^*)^{-1}$ remains lower triangular, any variable $X_i$ depends strictly on current and past noise terms. This structure ensures that the noise $Z_j$ is orthogonal to the predecessors $X_{\mathrm{pre}(j)}$. Thus, the bias term in the OLS estimator becomes zero. Consequently, the unique minimizer for each sub-problem exactly recovers the true structural parameters. The full proof can be found in Appendix \ref{theorem1}.

\section{Experiments}
\subsection{Experimental setup} 

\textbf{Baselines}
This study compares TriOpt with three state-of-the-art causal discovery methods, including three linear approaches (NOTEARS~\citep{zheng2018dags}, GOLEM~\citep{ng2020role}, DAGMA~\citep{bello2022dagma}).  
In our ablation studies, we also consider replacing the TriOpt ordering and postprocessing step with equivalents of the closely-related non-linear method CaPS~\citep{xu2024ordering}. 
We consider two different variants of TriOpt, with and without the Sherman-Morrison downdate.

\textbf{Synthetic Data} Our synthetic data generation protocol follows GOLEM \citep{ng2020role} and prior work \citep{zheng2018dags}. Ground-truth DAGs are generated from two random graph families: Erdős–Rényi (ER) graphs with edge probability $2e/(d^2-d)$ for $e\in{d,2d,4d}$ (ER1/ER2/ER4), and Scale-Free (SF) graphs generated via the Barabási–Albert model with $k=4$ (SF4). Edge weights are sampled uniformly from 

$[-2,-0.5]\cup[0.5,2]$.
We simulate $n\in\{2000,5000\}$ samples from linear SEMs with dimensionality $d\in\{100,200,500,1000\}$ and noise drawn from Gaussian, Gumbel, or Exponential distributions.

\textbf{Semi-synthetic Data} We evaluate TriOpt against baseline methods on two semi-synthetic datasets: \textbf{ARTH150}, a gene regulatory network benchmark from the bnlearn repository~\citep{jstatsoft09} based on \emph{Arabidopsis thaliana}~\citep{opgen2007correlation}, with $d=107$ nodes, $e=150$ edges, and $n=10{,}000$ samples generated using a linear structural equation model. We also consider \textbf{Syntren}~\citep{van2006syntren}, which consists of 10 simulated transcriptional networks, each with 500 observations and an underlying DAG with $d=20$ nodes and $e\in\{20,\ldots,25\}$ edges.

\textbf{Real World Data} We also ran experiments on a real world dataset
that measures the expression levels of proteins and phospholipids in human cells~\citep{sachs2005causal}. 
The dataset has \(d = 11\) variables (cell-signaling markers) and \(n = 853\) observational 
samples.

\paragraph{Metrics} For evaluation, the estimated graphs are assessed via the normalized Structural Hamming Distance (SHD), and F1 Score and ancestor Adjustment Identification Distance (AID) \citep{henckel2024adjustment}, with all metrics averaged over 10 independent random simulations.  
Unlike previous studies, we do not report the Structural Intervention Distance (SID). Because in higher dimensions, the computation time of SID is prohibitively high. AID serves as a practically relevant and scalable proxy of SID. The performance of ordering is reported by 
Order divergence~\citep{pmlr-v162-rolland22a} which measures the discrepancy between an estimated
topological ordering $\pi$ and the adjacency matrix of the true causal graph
$\mathcal{A}$. It is defined as
$D_{\mathrm{top}}(\pi, \mathcal{A}) = \sum_{i=1}^{d} \sum_{j : \pi_i > \pi_j} \mathcal{A}_{i,j}.$

\subsection{Results}
\label{main-results}
\textbf{Performance Over Different Graph types and noise}. We use the notation [graph
type][degree] for indicating experiments over different synthetic datasets. Figure \ref{fig:all_metrics_gev} shows the results of TriOpt over ER1, ER2, ER4 and SF4 graphs on equal variance gaussian noise type for metrics runtime, SHD, F1 score and AID on 5000 samples. The full results for all three noise types for 5000 and 2000 samples are reported in Figure \ref{fig:performance_5000} and \ref{fig:performance_2000} in Appendix \ref{full_result}. 
Figure \ref{vary-sparsity} in Appendix \ref{robustness-sparsity} shows how robust both variants of TriOpt are compared to baselines as sparsity increases. 

\begin{figure*}[htbp]
    \centering

    \begin{subfigure}{\linewidth}
        \centering
        \begin{minipage}[c]{0.72\linewidth}
            \centering
            \includegraphics[width=\linewidth]{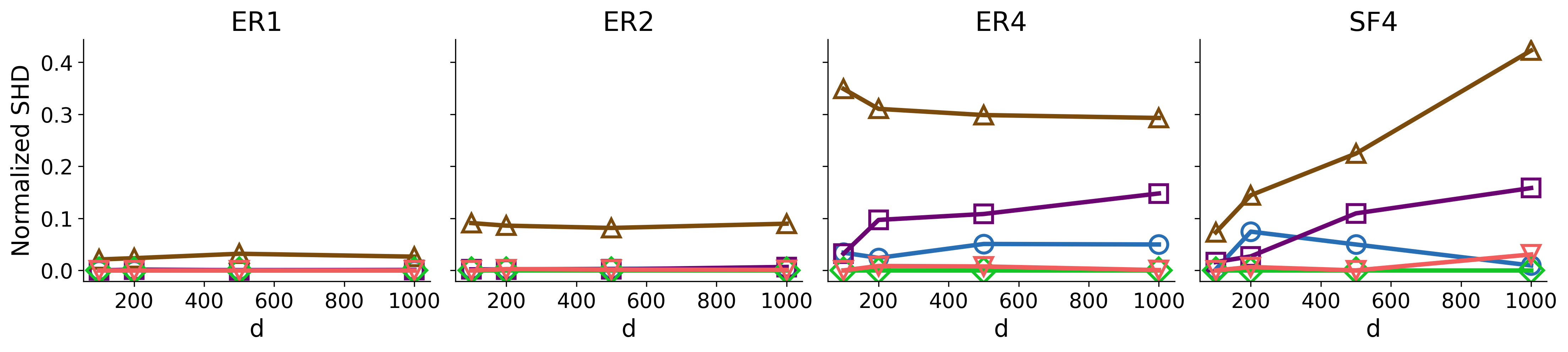}
        \end{minipage}
        \hfill
        \begin{minipage}[c]{0.23\linewidth}
            \caption{Normalized SHD}
            \label{fig:nshd_gev}
        \end{minipage}
    \end{subfigure}


    \begin{subfigure}{\linewidth}
        \centering
        \begin{minipage}[c]{0.72\linewidth}
            \centering
            \includegraphics[width=\linewidth]{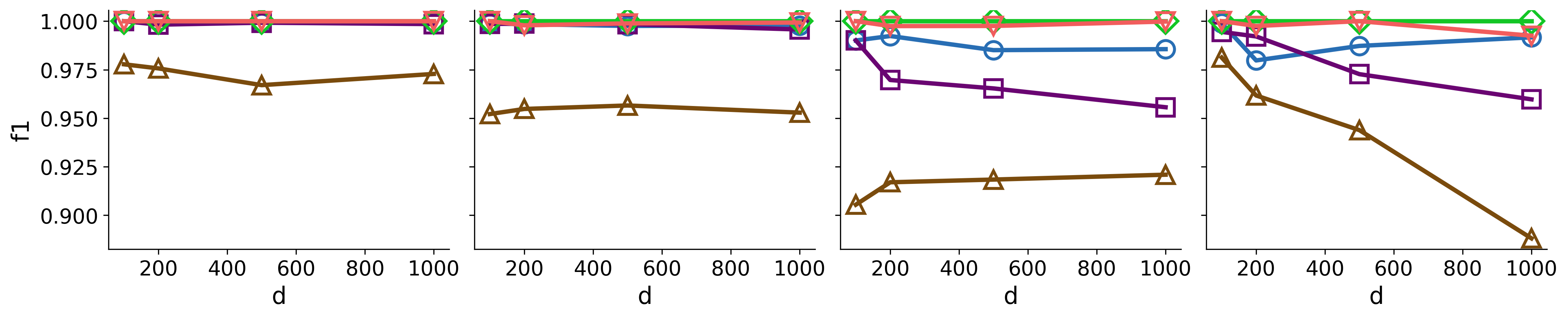}
        \end{minipage}
        \hfill
        \begin{minipage}[c]{0.23\linewidth}
            \caption{F1 Score}
            \label{fig:f1_gev}
        \end{minipage}
    \end{subfigure}


    \begin{subfigure}{\linewidth}
        \centering
        \begin{minipage}[c]{0.72\linewidth}
            \centering
            \includegraphics[width=\linewidth]{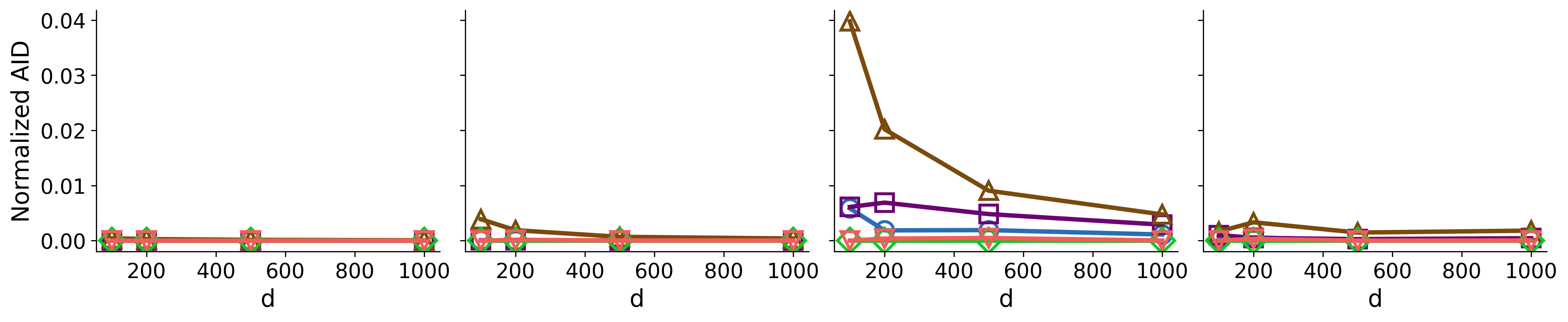}
        \end{minipage}
        \hfill
        \begin{minipage}[c]{0.23\linewidth}
            \caption{Normalized AID}
            \label{fig:naid_gev}
        \end{minipage}
    \end{subfigure}


    \begin{subfigure}{\linewidth}
        \centering
        \begin{minipage}[c]{0.72\linewidth}
            \centering
            \includegraphics[width=\linewidth]{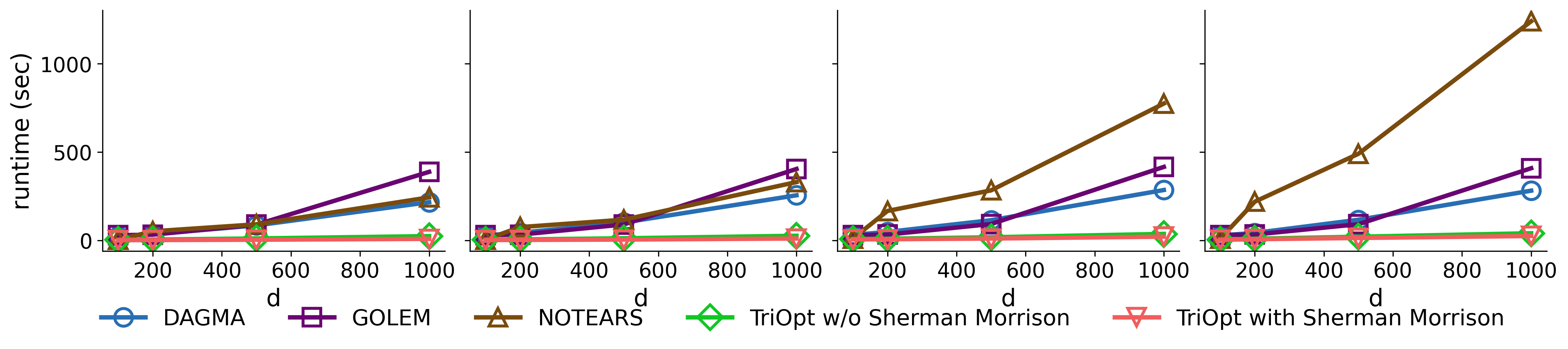}
        \end{minipage}
        \hfill
        \begin{minipage}[c]{0.23\linewidth}
            \caption{Runtime (sec)}
            \label{fig:runtime_gev}
        \end{minipage}
    \end{subfigure}

    \caption{%
        Performance comparison across graph types and degrees
        under \texttt{gaussian\_ev} noise with $n{=}5000$ samples.
        Each row of subplots corresponds to one metric;
        columns correspond to different graph/degree combinations.
    }
    \label{fig:all_metrics_gev}
    \vspace{-10pt}
\end{figure*}

In terms of runtime, both versions of TriOpt clearly outperform baselines (GOLEM, NOTEARS, DAGMA) by a significant margin. As the dimension increases, the difference in runtime becomes more drastic. For instance, in 1000 dimensions, TriOpt without Sherman Morrison is on average 95\% faster than NOTEARS, 92\% faster than GOLEM and 88\% faster than DAGMA. TriOpt with Sherman Morrison is on average 97\% faster than NOTEARS, 96\% faster than GOLEM and 93\% faster than DAGMA.
In terms of SHD, F1 and AID, TriOpt without Sherman morrison downdate has the best performance in all $48$ settings. With Sherman Morrison downdate the performance slightly deteriorates in Exponential, Gumbel noise cases. This is an expected trade-off, as these noise types are skewed and heavy-tailed, which makes Sherman Morrison downdate numerically inaccurate. 

Although the three novel improvements in the TriOpt framework were designed for scalability, two of them also contribute towards better accuracy. 
First, TriOpt's Upper Triangular Optimization in postprocessing intrinsically improves accuracy by rendering the optimization problem convex. Continuous optimization baselines (NOTEARS, DAGMA, GOLEM) need to enforce an acyclicity constraint. This leads to non-convex loss landscapes. Thanks to convexity, TriOpt guarantees convergence to a global minimum for this phase. Thus, TriOpt entirely avoids the local optima that frequently trap baseline methods. Second, using a linear kernel in the ordering phase gives a more accurate estimate of the causal ordering of the variables than an RBF kernel.

\textbf{Sensitivity Analysis of Weight Scale}.
Following \citep{ng2020role} \citep{zheng2018dags}, we study sensitivity to weight scaling on 200 node ER1, ER2, ER4 graphs, where edge weights are sampled uniformly from
$\alpha \cdot ([-2,-0.5] \cup [0.5,2])$
where $\alpha \in \{0.3,0.4,\ldots,1.0\}$. 
Figure \ref{sensitivity-shd} in appendix shows how the normalized SHD varies along different weight scales. It is evident that both variants of TriOpt are the best performant methods across all noise scales. GOLEM and NOTEARS gradually get more accurate as $\alpha$ increases, however, they become unstable for larger values of $\alpha$ (0.9 or 1.0).

\textbf{Sensitivity to Ordering Errors}
\label{sub-sens-ordering}.
We conducted further experiments demonstrating that TriOpt’s sensitivity to ordering errors depends on both the specific type of ordering error and the structure of the underlying graph. We also show the theoretical analysis of ordering mistakes for upper-triangular optimization with a sparsity penalty. The empirical results and theoretical analysis are provided in Appendix \ref{sec-sensitivity-ordering}. Empirical results show that across all three graph types, adjacent swaps and block reversals proved highly robust (F1 remaining at or near 1.000 even at 20\% perturbation), while ancestor swaps and hub swaps caused the most severe degradation - for example, a 20\% ancestor swap on SF4 dropped F1 by 0.698 and a 10\% hub swap on ER4 dropped F1 by 0.542, indicating that structurally meaningful ordering errors are far more damaging than positional ones.

\textbf{Misspecified Regime with Nonlinear Data}.
To demonstrate TriOpt's practicality on misspecified regime with nonlinear data, we ran additional experiments and generated nonlinear data following \citep{sanchez2023diffusion}'s protocol, considering ER1, ER4, SF1 and SF4 graphs. For each configuration, datasets of $ nodes \in \{25,50,100\} $ with $n=2000$ samples are generated. Figure \ref{fig-misspecified} in Appendix \ref{supplementary-results} show the results. In gaussian noise, TriOpt achieves lower Normalized SHD than DAGMA-Linear, GOLEM, and NOTEARS-Linear by approximately
\(3\%\), \(4\%\), and \(5\%\), respectively. It also improves F1 score over DAGMA-Linear, GOLEM, and NOTEARS-Linear by approximately
\(30\%\), \(10\%\), and \(76\%\), respectively.
The results demonstrate that TriOpt remains robust even when the assumption of linearity does not hold.  

\begin{table*}[t]
\centering
\scriptsize
\setlength{\tabcolsep}{3pt}

\begin{tabular}{lcccccc}
\toprule
& \multicolumn{2}{c}{\textbf{ARTH150}}
& \multicolumn{2}{c}{\textbf{Syntren}}
& \multicolumn{2}{c}{\textbf{Sachs}} \\
\cmidrule(lr){2-3} \cmidrule(lr){4-5} \cmidrule(lr){6-7}
\textbf{Model}
& \textbf{SHD} $\downarrow$
& \textbf{Runtime (s)} $\downarrow$
& \textbf{SHD} $\downarrow$
& \textbf{Runtime (s)} $\downarrow$
& \textbf{SHD} $\downarrow$
& \textbf{Runtime (s)} $\downarrow$ \\
\midrule

DAGMA
& $138.0 \pm 0.0$
& $41.20 \pm 0.40$
& $33.1 \pm 7.32$
& $33.12 \pm 3.46$
& $15.0 \pm 0.0$
& $22.109 \pm 0.031$ \\

GOLEM
& $163.0 \pm 0.0$
& $42.37 \pm 2.52$
& $43.9 \pm 11.23$
& $41.44 \pm 0.33$
& $20.4 \pm 1.3$
& $27.573 \pm 0.154$ \\

NOTEARS
& $140.4 \pm 0.89$
& $14.24 \pm 0.80$
& $\mathbf{33.0 \pm 5.66}$
& $51.82 \pm 1.21$
& $\mathbf{13.0 \pm 0.0}$
& $12.275 \pm 0.040$ \\

TriOpt w/o SM
& $\mathbf{135.0 \pm 0.0}$
& $6.95 \pm 0.05$
& $34.0 \pm 7.36$
& $1.55 \pm 1.64$
& $13.6 \pm 0.5$
& $0.717 \pm 0.475$ \\

TriOpt with SM
& $\mathbf{135.0 \pm 0.0}$
& $\mathbf{3.02 \pm 0.04}$
& $34.3 \pm 7.13$
& $\mathbf{1.02 \pm 0.10}$
& $14.0 \pm 0.0$
& $\mathbf{0.499 \pm 0.008}$ \\

\bottomrule
\end{tabular}
\caption{Performance comparison on ARTH150, Syntren, and Sachs based on SHD and runtime}
\label{tab:merged_shd_runtime}
\end{table*}

\textbf{Ablation study: Advantage of Linear Kernel and Sherman Morrison Downdate}.
To highlight TriOpt’s advantage in accelerating the topological ordering step, we performed an ablation study comparing three variants: (1) CaPS ordering, (2) TriOpt ordering with the Sherman–Morrison downdate, and (3) TriOpt ordering without the Sherman–Morrison downdate. The experiments were performed with nodes $ \in \{100, 200, 500,1000\}$ and $\text{number of samples}$ $\in \{2000,5000\} $.

Figure \ref{ablation-sherman-runtime-2000} in Appendix shows the comparison of ordering time, memory usage. It is evident that CaPS ordering incurs extremely high memory usage and thus becomes unusable when the number of dimension and number of samples is high. Figure \ref{ablation-sherman-odivergence-2000} in Appendix shows the comparison of order divergence over ER1, ER2, ER4 and SF4 graphs separately. On dense graphs (ER4 and SF4), CaPS ordering gets worse, wheras TriOpt keeps up perfect performance. In terms of runtime on 200 nodes, TriOpt, with and without the Sherman–Morrison downdate, is on average 114× and 22× faster, respectively, than CaPS ordering. Sherman morrison downdate was increasingly successful in reducing runtime as the dimension grew from 100 to 1000, without significant compromise in order divergence and peak memory usage. 
For 5000 samples, CaPS failed to run because of an out-of-memory error, while both variants of TriOpt executed successfully. The ordering time and order divergence results for 5000 samples are provided in Figure \ref{ablation-sherman-runtime-5000} and \ref{ablation-sherman-odivergence-5000} respectively Appendix \ref{ablation-sherman-5000}. 

\textbf{Ablation Study: Advantage of Triangular Optimization}.
To demonstrate the advantage of Upper triangular optimization over CAM pruning, we designed an ablation study with 2000 samples and nodes $\in \{20,50\} $. We reported the runtime, SHD and F1 score for (1) CAM Pruning (2) Triangular Optimization - on the ground truth ordering. Table \ref{ablation-ut} in Appendix \ref{appendix-advantage-ut} shows the comparison in terms of SHD and F1. Table \ref{ablation-ut-runtime} in Appendix \ref{appendix-advantage-ut} shows the comparison in terms of optimization/pruning time. On average, TriOpt’s optimization with an upper triangular constraint is 26 times faster than CAM pruning on 20-node problems and 664 times faster on 50-node problems, while also achieving better performance.

\textbf{Ablation Study: CaPS with Linear Kernel}. 
We ran experiments with CaPS, replacing its RBF kernel with a linear kernel but keeping its pruning step. Due to the scalability bottleneck of CaPS, we are considering dimensions $d \in \{50\}$. The results in table \ref{tab:linearCaPS} in Appendix show that TriOpt is still orders of magnitude faster than the linearized CaPS, and comparable or better in terms of accuracy.

\textbf{Results on Semi-synthetic and Real World Data}. From Table \ref{tab:merged_shd_runtime} it is evident that TriOpt achieves the best SHD on ARTH150 and remains highly competitive on Syntren and Sachs (within margin of the best), while delivering order of magnitude faster runtime. Similar trends can be seen in terms of F1 score and AID - which are reported in Table \ref{tab:merged_aid_f1} in appendix. The key takeaway is that when linearity assumptions are approximately satisfied in the real-world, TriOpt offers a scalable solution while maintaining competitive accuracy.

 \section{Conclusion and Future Work} 
We have introduced TriOpt, a novel framework which accelerates the runtime and improves the accuracy of linear DAG learning by threefold optimizations (a) avoiding kernel recomputation via exploiting the additivity property of linear kernel (b) avoiding kernel inverse recomputation by Sherman Morrison downdate (c) optimizing only the strictly upper triangular matrix without the acyclicity constraint, rendering the optimization problem convex.
We presented two versions of TriOpt with and without Sherman Morrison downdate - where the former offers reduction in ordering time especially on high dimensions. 
Results on synthetic, semi-synthetic and real world data indicate that TriOpt significantly improves runtime over existing baselines, with accuracy that is comparable or superior. One limitation of this framework is that the Sherman–Morrison downdate can occasionally become numerically unstable under Exponential and Gumbel noise, particularly when the denominator gets close to zero. Although we address this using periodic re-inversion and threshold-based re-inversion, future work could investigate more advanced strategies for stabilizing the kernel inverse updates. Future work can also explore how to extend TriOpt framework when the ground truth causal ordering is partially known. 

\bibliography{reffile.bib}
\bibliographystyle{plainnat}

\appendix
\onecolumn

\section{Proof of theorem 1}
\label{theorem1}

\begin{proof}
\textbf{1. Well-posedness of the covariance.}

As $B^*$ is strictly lower triangular, it is nilpotent, and the matrix $(I - B^*)$ is invertible with determinant 1. The SEM can be written as:
\[
X = (I - B^*)^{-1} Z = A Z,
\]
where $A := (I - B^*)^{-1}$. As the inverse of a lower triangular matrix is lower triangular with unit diagonal, $A$ is full rank. The covariance of $X$ can be written as:
\[
\Sigma_X = \mathbb{E}[XX^T] = A \mathbb{E}[ZZ^T] A^T = A \Omega A^T.
\]
Since $\Omega$ is positive definite (by the third assumption: $\omega_j^2 > 0$) and $A$ is non-singular, $\Sigma_X$ is positive definite. Therefore, any principal submatrix $\Sigma_{\mathrm{pre}(j)} := \mathbb{E}[X_{\mathrm{pre}(j)} X_{\mathrm{pre}(j)}^T]$ is also positive definite and invertible.

\medskip

\medskip
\noindent\textbf{2. Separable Least Squares.}

As the $j$-th column of $W$ depends only on $X_{\mathrm{pre}(j)}$, the problem decomposes into $d$ independent sub-problems:
\[
\min_{w_j} J_j(w_j) = \frac{1}{2} \mathbb{E} \left[ (X_j - w_j^T X_{\mathrm{pre}(j)})^2 \right].
\]
Expanding \( J_j(w_j) \), we obtain:
\[
\begin{aligned}
J_j(w_j)
&= \frac{1}{2}\,\mathbb{E}\!\left[
X_j^2
- 2 X_j\, w_j^{\top} X_{\mathrm{pre}(j)}
+ w_j^{\top} X_{\mathrm{pre}(j)} X_{\mathrm{pre}(j)}^{\top} w_j
\right] \\
&= \frac{1}{2}\left(
\mathbb{E}[X_j^2]
- 2 w_j^{\top} \mathbb{E}[X_{\mathrm{pre}(j)} X_j]
+ w_j^{\top} \Sigma_{\mathrm{pre}(j)} w_j
\right).
\end{aligned}
\]

This is a quadratic function of \( w_j \).
Its Hessian is given by
\[
\nabla^2 J_j(w_j) = \Sigma_{\mathrm{pre}(j)}.
\]

From Step~1, \( \Sigma_{\mathrm{pre}(j)} \succ 0 \), and therefore
\( J_j(w_j) \) is strictly convex and admits a unique global minimizer given by the OLS solution:
\[
\hat{w}_j = \Sigma_{\mathrm{pre}(j)}^{-1} \mathbb{E}[X_{\mathrm{pre}(j)} X_j].
\]

\medskip
\noindent\textbf{3. Recovery of the DAG}

From the structural equation for coordinate $j$, we have:
\[
X_j = \sum_{k=1}^{j-1} B^*_{jk} X_k + Z_j = (w^*_j)^T X_{\mathrm{pre}(j)} + Z_j,
\]
where $w^*_j = (B^*_{j1}, \dots, B^*_{j, j-1})^T$ corresponds to the non-zero entries of the $j$-th row of $B^*$. Substituting this into the OLS estimator:

\begin{align*}
 \hat{w}_j &= \Sigma_{\mathrm{pre}(j)}^{-1} \mathbb{E}\left[ X_{\mathrm{pre}(j)} \left( (w^*_j)^T X_{\mathrm{pre}(j)} + Z_j \right) \right] \\
 &= \Sigma_{\mathrm{pre}(j)}^{-1} \left( \mathbb{E}[X_{\mathrm{pre}(j)} X_{\mathrm{pre}(j)}^T] w^*_j + \mathbb{E}[X_{\mathrm{pre}(j)} Z_j] \right) \\
 &= \Sigma_{\mathrm{pre}(j)}^{-1} \Sigma_{\mathrm{pre}(j)} w^*_j + \Sigma_{\mathrm{pre}(j)}^{-1} \mathbb{E}[X_{\mathrm{pre}(j)} Z_j] \\
 &= w^*_j + \Sigma_{\mathrm{pre}(j)}^{-1} \mathbb{E}[X_{\mathrm{pre}(j)} Z_j]
\end{align*}

If we can prove that $\mathbb{E}[X_{\mathrm{pre}(j)} Z_j] = 0$, then the true adjacency matrix will be recovered exactly. 

\noindent\textbf{4. Deriving Orthogonality}

Using the reduced form $X = AZ$, the $i$-th variable is:
\[
X_i = \sum_{k=1}^{d} A_{ik} Z_k.
\]

Since $A$ is lower triangular, $A_{ik} = 0$ for $k > i$. Thus $X_i$ depends only on noise terms $Z_1, \dots, Z_i$:
\[
X_i = \sum_{k=1}^{i} A_{ik} Z_k.
\]

For the covariance with $Z_j$ where $j > i$:
\[
\mathbb{E}[X_i Z_j] = \mathbb{E} \Bigg[ \Bigg( \sum_{k=1}^{i} A_{ik} Z_k \Bigg) Z_j \Bigg] 
= \sum_{k=1}^{i} A_{ik} \mathbb{E}[Z_k Z_j].
\]

By the structural independence assumption, $\mathbb{E}[Z_k Z_j] = 0$ whenever $k \neq j$. Since the summation index $k$ runs from $1$ to $i$, and $i < j$, we have $k \neq j$ for all terms in the sum. Therefore:
\[
\mathbb{E}[X_i Z_j] = 0 \quad \text{for all } i < j.
\]

Therefore, 
\[
\mathbb{E}[X_{\mathrm{pre}(j)} Z_j] = 0.
\]

Using that in Step 3, 

\begin{align*}
 \hat{w}_j &=  w^*_j + \Sigma_{\mathrm{pre}(j)}^{-1} 0 \\ 
 &= w^*_j
\end{align*}

Thus, the minimizer for each subproblem recovers the true parameters of the structural equation. Therefore, the global minimizer is 
\[
W^* = (B^*)^\top.
\]

\end{proof}

Note that, we formulate Theorem 4.1 without the $\ell_1$ penalty. This is because, in the population limit, regularization is unnecessary for exact recovery. However, this changes in empirical setting. According to \citep{ng2020role} because of random estimation errors inherent to finite datasets, linear coefficients that are truly zero in the underlying structural equation model will often yield small, non-zero estimated values. The sparsity penalty forces these small, noisy edge weights to zero. Thus, the inclusion of the $\ell_1$ penalty in the practical algorithm is a necessary bridge from the population to finite-sample empirical data. This is consistent with standard practices in continuous optimization frameworks such as NOTEARS \citep{zheng2018dags} , GOLEM \citep{ng2020role}, DAGMA \citep{bello2022dagma}.

\textbf{Practical Implications of Theorem 4.1:} Although Theorem 4.1 assumes a known ground-truth causal ordering, it remains practically relevant because such orderings are naturally available in domains like biological pathways and industrial pipelines, while in cases where they are unknown, Algorithm \ref{alg:stein_ordering} can be used to estimate the causal order.

\section{Sensitivity to Ordering Mistakes}
\label{sec-sensitivity-ordering}
\textbf{Empirical Results:} 
We designed 5 types of ordering perturbations — random shuffle (unstructured noise), block reversal (contiguous segment flip), adjacent swap (minimal realistic error), ancestor swap (real edge inversion), and hub swap (high-degree targeting) and report on F1 score.
We ran the experiments on ER1, ER4 and SF4 graphs with 100 nodes and 5000 samples.  Table \ref{tab:er_degree1_results}, \ref{tab:er_degree4_results}, \ref{tab:sf_degree4_results} shows the experiments results respectively.

\begin{table*}[H]
\centering
\caption{Results for dataset ER with degree 1. Baseline (ground truth ordering) SHD: 0.000 | F1: 1.000.}
\label{tab:er_degree1_results}
\begin{tabular}{lccccc}
\toprule
\textbf{Method} & \textbf{\%} & \textbf{SHD} & \textbf{$\Delta$ SHD} & \textbf{F1} & \textbf{$\Delta$ F1} \\ \midrule
\multirow{4}{*}{random shuffle} 
& 1\% & 0.000 $\pm$ 0.000 & (+0.000) & 1.000 $\pm$ 0.000 & (+0.000) \\
& 5\% & 2.800 $\pm$ 1.720 & (+2.800) & 0.963 $\pm$ 0.019 & (-0.037) \\
& 10\% & 6.200 $\pm$ 2.786 & (+6.200) & 0.925 $\pm$ 0.032 & (-0.075) \\
& 20\% & 17.400 $\pm$ 8.868 & (+17.400) & 0.815 $\pm$ 0.073 & (-0.185) \\
\midrule
\multirow{4}{*}{block reversal} 
& 1\% & 0.000 $\pm$ 0.000 & (+0.000) & 1.000 $\pm$ 0.000 & (+0.000) \\
& 5\% & 0.000 $\pm$ 0.000 & (+0.000) & 1.000 $\pm$ 0.000 & (+0.000) \\
& 10\% & 0.000 $\pm$ 0.000 & (+0.000) & 1.000 $\pm$ 0.000 & (+0.000) \\
& 20\% & 0.000 $\pm$ 0.000 & (+0.000) & 1.000 $\pm$ 0.000 & (+0.000) \\
\midrule
\multirow{4}{*}{adjacent swap} 
& 1\% & 0.000 $\pm$ 0.000 & (+0.000) & 1.000 $\pm$ 0.000 & (+0.000) \\
& 5\% & 0.000 $\pm$ 0.000 & (+0.000) & 1.000 $\pm$ 0.000 & (+0.000) \\
& 10\% & 0.000 $\pm$ 0.000 & (+0.000) & 1.000 $\pm$ 0.000 & (+0.000) \\
& 20\% & 0.000 $\pm$ 0.000 & (+0.000) & 1.000 $\pm$ 0.000 & (+0.000) \\
\midrule
\multirow{4}{*}{ancestor swap} 
& 1\% & 1.400 $\pm$ 0.800 & (+1.400) & 0.979 $\pm$ 0.007 & (-0.021) \\
& 5\% & 11.800 $\pm$ 3.970 & (+11.800) & 0.858 $\pm$ 0.033 & (-0.142) \\
& 10\% & 32.200 $\pm$ 8.328 & (+32.200) & 0.671 $\pm$ 0.065 & (-0.329) \\
& 20\% & 49.600 $\pm$ 6.888 & (+49.600) & 0.505 $\pm$ 0.057 & (-0.495) \\
\midrule
\multirow{4}{*}{hub swap} 
& 1\% & 9.600 $\pm$ 5.161 & (+9.600) & 0.907 $\pm$ 0.049 & (-0.093) \\
& 5\% & 18.000 $\pm$ 13.799 & (+18.000) & 0.837 $\pm$ 0.113 & (-0.163) \\
& 10\% & 35.400 $\pm$ 18.337 & (+35.400) & 0.680 $\pm$ 0.140 & (-0.320) \\
& 20\% & 39.600 $\pm$ 5.083 & (+39.600) & 0.634 $\pm$ 0.045 & (-0.366) \\
\bottomrule
\end{tabular}
\end{table*}

\begin{table}[H]
\centering
\caption{Results for dataset ER with degree 4. Baseline (ground truth ordering) SHD: 0.000 | F1: 1.000.}
\label{tab:er_degree4_results}
\begin{tabular}{lccccc}
\toprule
\textbf{Method} & \textbf{\%} & \textbf{SHD} & \textbf{$\Delta$ SHD} & \textbf{F1} & \textbf{$\Delta$ F1} \\ \midrule
\multirow{4}{*}{random shuffle} 
& 1\% & 0.000 $\pm$ 0.000 & (+0.000) & 1.000 $\pm$ 0.000 & (+0.000) \\
& 5\% & 35.200 $\pm$ 41.344 & (+35.200) & 0.924 $\pm$ 0.081 & (-0.076) \\
& 10\% & 99.000 $\pm$ 32.330 & (+99.000) & 0.794 $\pm$ 0.055 & (-0.206) \\
& 20\% & 191.800 $\pm$ 51.375 & (+191.800) & 0.637 $\pm$ 0.077 & (-0.363) \\
\midrule
\multirow{4}{*}{block reversal} 
& 1\% & 0.000 $\pm$ 0.000 & (+0.000) & 1.000 $\pm$ 0.000 & (+0.000) \\
& 5\% & 1.600 $\pm$ 3.200 & (+1.600) & 0.996 $\pm$ 0.008 & (-0.004) \\
& 10\% & 18.800 $\pm$ 16.570 & (+18.800) & 0.956 $\pm$ 0.039 & (-0.044) \\
& 20\% & 77.200 $\pm$ 58.438 & (+77.200) & 0.835 $\pm$ 0.111 & (-0.165) \\
\midrule
\multirow{4}{*}{adjacent swap} 
& 1\% & 0.000 $\pm$ 0.000 & (+0.000) & 1.000 $\pm$ 0.000 & (+0.000) \\
& 5\% & 0.000 $\pm$ 0.000 & (+0.000) & 1.000 $\pm$ 0.000 & (+0.000) \\
& 10\% & 0.000 $\pm$ 0.000 & (+0.000) & 1.000 $\pm$ 0.000 & (+0.000) \\
& 20\% & 0.000 $\pm$ 0.000 & (+0.000) & 1.000 $\pm$ 0.000 & (+0.000) \\
\midrule
\multirow{4}{*}{ancestor swap} 
& 1\% & 18.400 $\pm$ 12.659 & (+18.400) & 0.956 $\pm$ 0.027 & (-0.044) \\
& 5\% & 111.400 $\pm$ 38.422 & (+111.400) & 0.770 $\pm$ 0.060 & (-0.230) \\
& 10\% & 202.400 $\pm$ 37.382 & (+202.400) & 0.618 $\pm$ 0.052 & (-0.382) \\
& 20\% & 303.600 $\pm$ 42.580 & (+303.600) & 0.465 $\pm$ 0.044 & (-0.535) \\
\midrule
\multirow{4}{*}{hub swap} 
& 1\% & 33.200 $\pm$ 26.551 & (+33.200) & 0.926 $\pm$ 0.057 & (-0.074) \\
& 5\% & 155.800 $\pm$ 57.638 & (+155.800) & 0.699 $\pm$ 0.102 & (-0.301) \\
& 10\% & 324.800 $\pm$ 43.902 & (+324.800) & 0.458 $\pm$ 0.056 & (-0.542) \\
& 20\% & 262.000 $\pm$ 25.938 & (+262.000) & 0.533 $\pm$ 0.028 & (-0.467) \\
\bottomrule
\end{tabular}
\end{table}

\begin{table}[ht]
\centering
\caption{Results for dataset SF with degree 4. Baseline (ground truth ordering) SHD: 0.000 | F1: 1.000.}
\label{tab:sf_degree4_results}
\begin{tabular}{lccccc}
\toprule
\textbf{Method} & \textbf{\%} & \textbf{SHD} & \textbf{$\Delta$ SHD} & \textbf{F1} & \textbf{$\Delta$ F1} \\ \midrule
\multirow{4}{*}{random shuffle} 
& 1\%  & 0.000 $\pm$ 0.000 & (+0.000) & 1.000 $\pm$ 0.000 & (+0.000) \\
& 5\%  & 29.400 $\pm$ 47.395 & (+29.400) & 0.932 $\pm$ 0.102 & (-0.068) \\
& 10\% & 58.000 $\pm$ 60.808 & (+58.000) & 0.870 $\pm$ 0.125 & (-0.130) \\
& 20\% & 182.600 $\pm$ 112.253 & (+182.600) & 0.667 $\pm$ 0.187 & (-0.333) \\ \midrule
\multirow{4}{*}{block reversal} 
& 1\%  & 0.000 $\pm$ 0.000 & (+0.000) & 1.000 $\pm$ 0.000 & (+0.000) \\
& 5\%  & 0.000 $\pm$ 0.000 & (+0.000) & 1.000 $\pm$ 0.000 & (+0.000) \\
& 10\% & 2.000 $\pm$ 4.000 & (+2.000) & 0.994 $\pm$ 0.012 & (-0.006) \\
& 20\% & 0.000 $\pm$ 0.000 & (+0.000) & 1.000 $\pm$ 0.000 & (+0.000) \\ \midrule
\multirow{4}{*}{adjacent swap} 
& 1\%  & 0.000 $\pm$ 0.000 & (+0.000) & 1.000 $\pm$ 0.000 & (+0.000) \\
& 5\%  & 0.000 $\pm$ 0.000 & (+0.000) & 1.000 $\pm$ 0.000 & (+0.000) \\
& 10\% & 8.400 $\pm$ 16.800 & (+8.400) & 0.980 $\pm$ 0.041 & (-0.020) \\
& 20\% & 16.600 $\pm$ 20.333 & (+16.600) & 0.960 $\pm$ 0.049 & (-0.040) \\ \midrule
\multirow{4}{*}{ancestor swap} 
& 1\%  & 101.400 $\pm$ 61.924 & (+101.400) & 0.791 $\pm$ 0.110 & (-0.209) \\
& 5\%  & 237.200 $\pm$ 91.951 & (+237.200) & 0.580 $\pm$ 0.123 & (-0.420) \\
& 10\% & 346.200 $\pm$ 34.102 & (+346.200) & 0.432 $\pm$ 0.055 & (-0.568) \\
& 20\% & 478.200 $\pm$ 32.799 & (+478.200) & 0.302 $\pm$ 0.032 & (-0.698) \\ \midrule
\multirow{4}{*}{hub swap} 
& 1\%  & 39.200 $\pm$ 61.950 & (+39.200) & 0.918 $\pm$ 0.122 & (-0.082) \\
& 5\%  & 236.800 $\pm$ 79.592 & (+236.800) & 0.589 $\pm$ 0.107 & (-0.411) \\
& 10\% & 402.600 $\pm$ 50.230 & (+402.600) & 0.403 $\pm$ 0.048 & (-0.597) \\
& 20\% & 306.600 $\pm$ 68.864 & (+306.600) & 0.506 $\pm$ 0.074 & (-0.494) \\
\bottomrule
\end{tabular}
\end{table}

TriOpt is robust to mild structural errors (adjacent swaps cause near-zero degradation even at 20\%, while random and block perturbations hurt performance beyond 10\%). However, perturbations targeting ancestor and hub swaps cause rapid degradation even at 1\%. Perturbation sensitivity is higher in denser graphs (ER4,SF4). SF4 degrades the most under ancestor and hub swaps, owing to its skewed degree distribution.

\textbf{On sparsity penalty:} The sparsity penalty can correct spurious edges (false positives). If the ordering suggests that $X_i$ could be a parent of $X_j$, but the data does not support a functional relationship, the sparsity penalty will drive the weight $W_{ij}$ to zero.  However, it cannot correct false negatives. Any relationship that would require a non-zero entry in the lower triangle $(i \ge  j)$ is impossible to recover because those are fixed at zero by the hard constraint. 

\textbf{On error propagation.}
Suppose \( i \in \mathrm{Pa}_G(j) \), meaning \( B^\ast_{ij} \neq 0 \), but the estimated ordering satisfies \( \hat{\pi}(j) < \hat{\pi}(i) \).
In this case, the mask enforces \( W_{ij} = 0 \) deterministically, so \( \hat{B}_{ij} = 0 \), and the edge is structurally removed.

Mathematically, the OLS objective for the \( j \)-th subproblem is
\[
J_j(w_j)
= \frac{1}{2} \, \mathbb{E}\!\left[ \left( X_j - w_j^\top X_{\hat{\mathrm{pre}}(j)} \right)^2 \right],
\]
where \( \hat{\mathrm{pre}}(j) \) denotes the set of nodes preceding \( j \) in the estimated ordering \( \hat{\pi} \).

If \( i \notin \hat{\mathrm{pre}}(j) \) but \( i \to j \) is a true edge, then \( X_i \) is excluded from the regression. Since
\[
X_j
= B^\ast_{ij} X_i
+ \sum_{k \in \mathrm{Pa}_G(j) \setminus \{i\}} B^\ast_{kj} X_k
+ Z_j,
\]
the contribution \( B^\ast_{ij} X_i \) is unmodeled and becomes part of an effective noise term:
\[
\tilde{Z}_j = B^\ast_{ij} X_i + Z_j.
\]

This breaks the orthogonality condition required by Theorem 4.1:
\[
\mathbb{E}\!\left[ X_{\hat{\mathrm{pre}}(j)} \, \tilde{Z}_j \right] = 0,
\]
because \( X_i \) is generally a linear combination of ancestors of \( i \), some of which may lie in \( \hat{\mathrm{pre}}(j) \).

If any such ancestor \( k \in \hat{\mathrm{pre}}(j) \), then
\[
\mathbb{E}[X_k \tilde{Z}_j]
= B^\ast_{ij} \, \mathbb{E}[X_k X_i] \neq 0,
\]
in general. This induces omitted-variable bias in all coefficients involving shared ancestors.


\section{Computational Complexity}
\label{computational-complexity}
\paragraph{Ordering Phase} Table ~\ref{timecomplexity} compares the time complexity of TriOpt operations with corresponding CaPS~\citep{xu2024ordering} operations. Here, n = number of samples, d = number of features, k = number of times to reinvert for numerical stability (which is much smaller than d). 
\begin{table}[h!]
\centering
\begin{tabular}{c|c|c}
\hline
\textbf{Operation} & \textbf{CaPS} 
                   & \textbf{TriOpt} \\ \hline\hline

Kernel Choice 
& RBF
& Linear 
\\ \hline

Kernel construction 
& $O(n^{2} d^{2})$
& $O(n^{2} kd)$ 
\\ \hline

Gradient term $\nabla K$ 
& $O(n^{2} d^{2})$
& $O(n d^{2})$
\\ \hline

Kernel inverse $K^{-1}$ 
& $O(d\, n^{3})$
& $O(kn^{3} + d\, n^{2})$
\\ \hline

\end{tabular}
\caption{Time complexity comparison of RBF vs Linear Stein kernel computations 
over all $d$ feature-removal iterations.}
\label{timecomplexity}
\end{table}

\paragraph{Optimization Phase}   
Computing the Gram matrix $G=\frac{1}{n}X^\top X$ requires $O(nd^2)$ time and is performed once.
Each optimization step is dominated by evaluating and differentiating
$\mathrm{Tr}(W^\top G W)$, which costs $O(d^3)$, while the linear trace term,
smooth $\ell_1$ regularization, and masking cost $O(d^2)$.
For $T$ iterations, the total complexity is $O(nd^2) + O(Td^3)$. Although TriOpt Optimization belongs to the same $O(d^3)$ complexity  class as NOTEARS and other continuous optimization methods, TriOpt's $O(d^3)$ complexity comes from pure matrix multiplications, whereas for other continuous optimization methods, it comes from computing matrix exponentials. \citet{higham2005scaling} demonstrates that calculating the matrix exponential is algorithmically equivalent to performing several matrix multiplication passes. 
Furthermore, TriOpt Optimization objective is convex, while continuous optimization based methods have a non-convex optimization objective. Consequently, TriOpt not only guarantees global optimality but also achieves convergence much faster. These theoretical advantages are empirically validated by the results in Section \ref{main-results}.

\section{Experiment Details}
\label{experiment-details}
\subsection{Baseline Implementation Details}
The implementation details of the baselines are listed below: 
\begin{enumerate}
    \item NOTEARS \citep{zheng2018dags}: The original author code (\hyperlink{https://github.com/xunzheng/notears}{https://github.com/xunzheng/notears}) does not have a GPU implementation. In order to make a fair comparison in terms of runtime, we developed a GPU implementation of NOTEARS in Pytorch. The choice of hyperparameter $\lambda_1$ was set to = 0.1 following \citep{ng2020role}. As threshold we went with the default $w=0.3$. 

    \item GOLEM \citep{ng2020role}: We employed the GOLEM implementation from the Gcastle \citep{zhang2021gcastle} Python package, using its default hyperparameter configuration.  

    \item DAGMA \citep{bello2022dagma}: The original author code (\hyperlink{https://github.com/kevinsbello/dagma}{https://github.com/kevinsbello/dagma}) does not have a GPU implementation for the linear version. In order to make a fair comparison in terms of runtime, we developed a GPU implementation of DAGMA-Linear in Pytorch. The choice of hyperparameters was exactly the same as in the original paper. 

    \item CaPS \citep{xu2024ordering}: We used original author code (\hyperlink{https://github.com/E2real/CaPS}{https://github.com/E2real/CaPS}). All CaPS experiments were run with the default pre-pruning and edge supplement. The choice of hyperparameters was exactly the same as in the original paper. 
    
\end{enumerate} 

In the experiments, we use default hyperparameters for these baselines unless otherwise stated. All experiments were conducted on a server with 503 GB RAM, dual AMD EPYC 7513 processors (64 cores total), and an NVIDIA A100-SXM4 GPU (80 GB) configured with Multi-Instance GPU (MIG).

\subsection{TriOpt Implementation Details}

\subsubsection{Ordering Phase}
Our implementation of the ordering is based on CaPS's open source implementation which can be found at (\hyperlink{https://github.com/E2real/CaPS}{https://github.com/E2real/CaPS}).
In the ordering phase, the choice of hyperparameters was as follows : 

 \begin{table}[H]
\centering
\begin{tabular}{|l|l|p{10cm}|}
\hline
\textbf{HyperParameter} & \textbf{ Value} & \textbf{Description} \\
\hline
$\eta_{G}$ & $0.9$ & Regularization parameter for the Kernel matrix.  \\
\hline

$\tau$ & 100 & Frequency of periodic Kernel matrix reinversion for numerical stability. 
\\
\hline
$\zeta$ & $10^{-7}$ & Threshold for matrix reinversion for numerical stability. 
\\
\hline
\end{tabular}
\caption{Hyperparameters for the Ordering phase}
\label{tab:hyperparameters}
\end{table}

k is defined as $\frac{d}{\tau}$.  

\subsubsection{Optimization Phase}
In the optimization phase, the choice of hyperparameters was as follows : 

\begin{table}

\centering
\begin{tabular}{|l |c |}
\hline
\textbf{Hyperparameter name} & \textbf{Value} \\
\hline
$\lambda_1$      & 0.001 \\
\texttt{max\_iter}      & 100 \\
\texttt{w\_threshold}   & 0.3 \\
$\epsilon$       & 1\text{e-}4 \\
\hline
\end{tabular}
\caption{Optimization Phase Hyperparameters}
\end{table}

\section{Supplementary Experiment Results}
\label{supplementary-results}

\subsection{Misspecified Regime with Nonlinear Data}
\begin{figure}[H]
    \centering
    \begin{subfigure}[b]{0.6\linewidth}
        \centering
        \includegraphics[width=\linewidth]{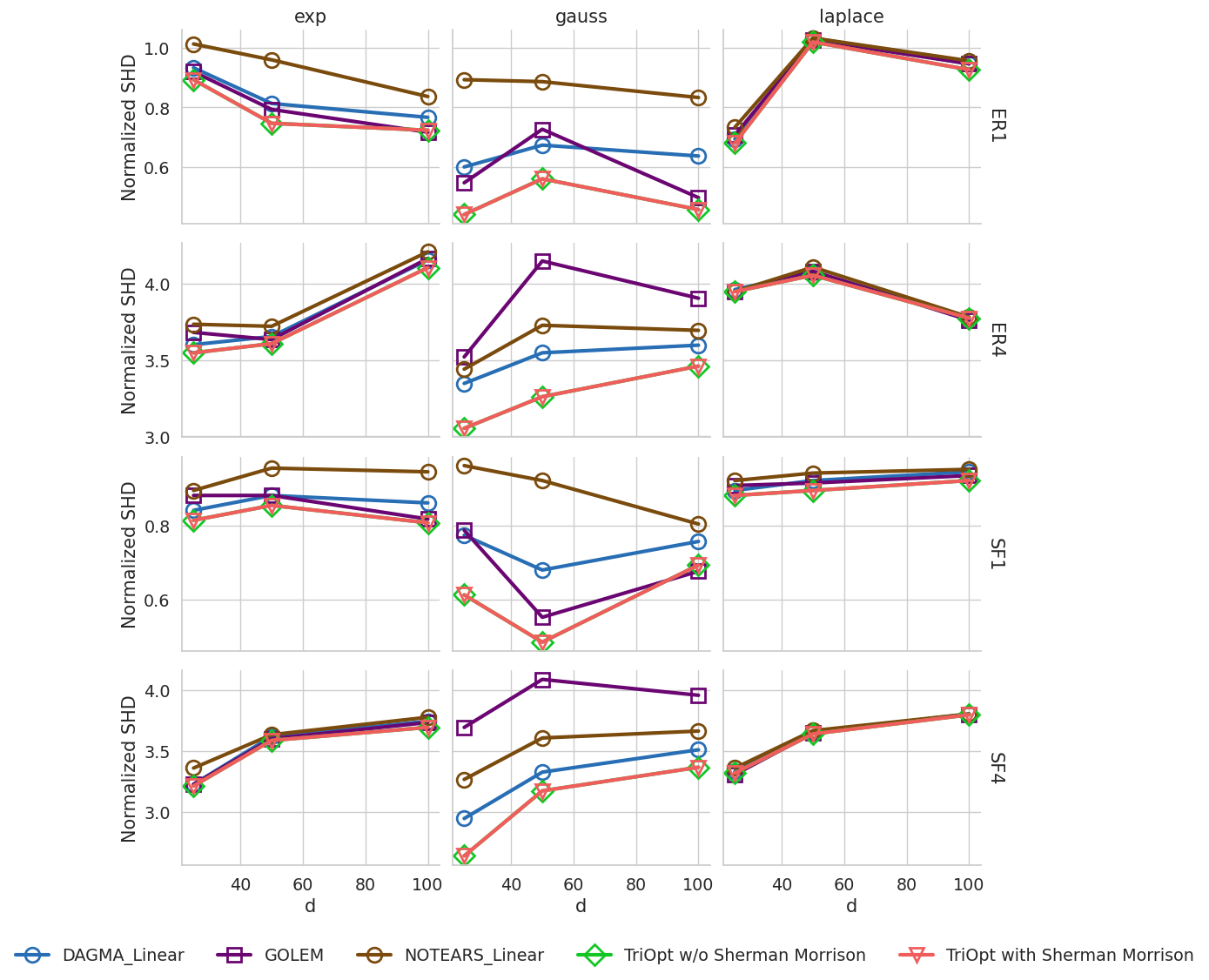}
        \caption{SHD}
        \label{fig:misspecified-shd}
    \end{subfigure}\hfill
    \begin{subfigure}[b]{0.65\linewidth}
        \centering
        \includegraphics[width=\linewidth]{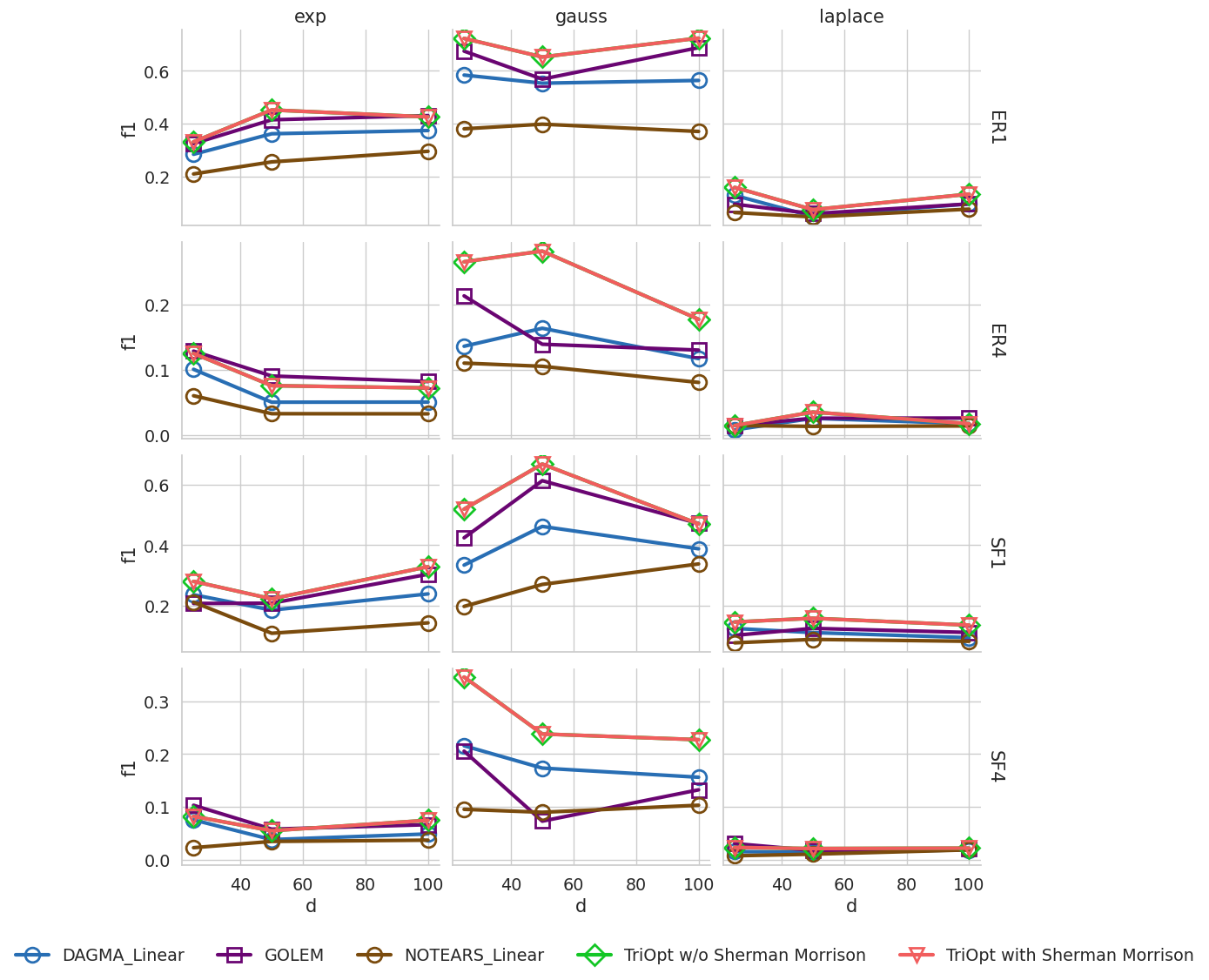}
        \caption{F1}
        \label{fig:misspecified-f1}
    \end{subfigure}
    \caption{Results on Nonlinear Data. TriOpt stays robust, delivering accuracy that is better (exponential, gaussian noise) than or on par (laplacian noise) with linear baselines, even when the assumption of linearity is violated.}
    \label{fig-misspecified}
\end{figure}

\subsection{Sensitivity Analysis of Weight Scale}

\begin{figure}[H]
    \centering
    \includegraphics[width=0.85\linewidth]{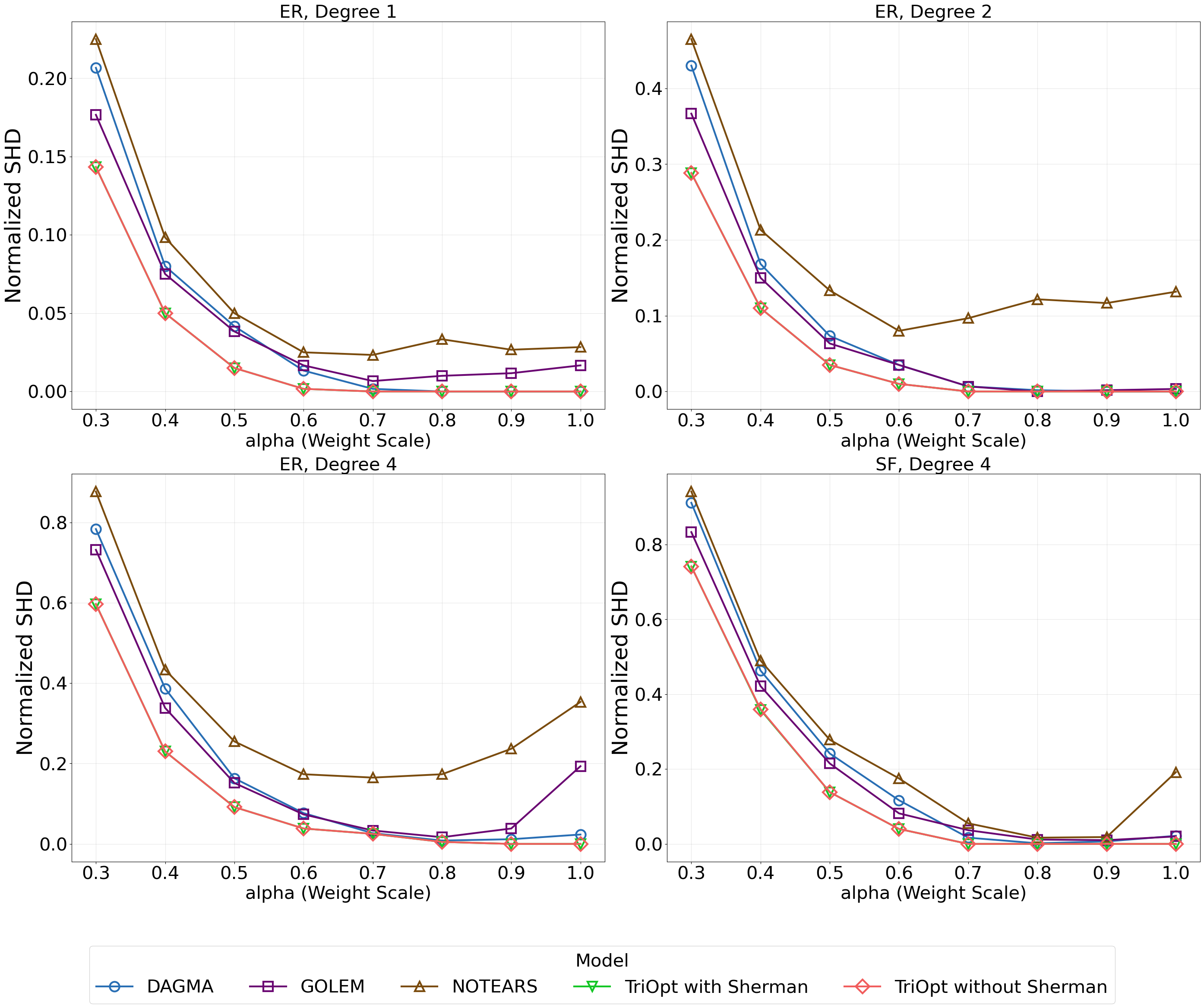}
    \caption{Normalized SHD (Lower Better) vs $\alpha$. TriOpt consistently demonstrates stronger robustness across the entire weight range than the baseline methods}
    \label{sensitivity-shd}
\end{figure}

\clearpage

\subsection{Robustness varying sparsity results}
\label{robustness-sparsity}

\begin{figure}[H]
    \centering
    \includegraphics[width=0.90\linewidth]{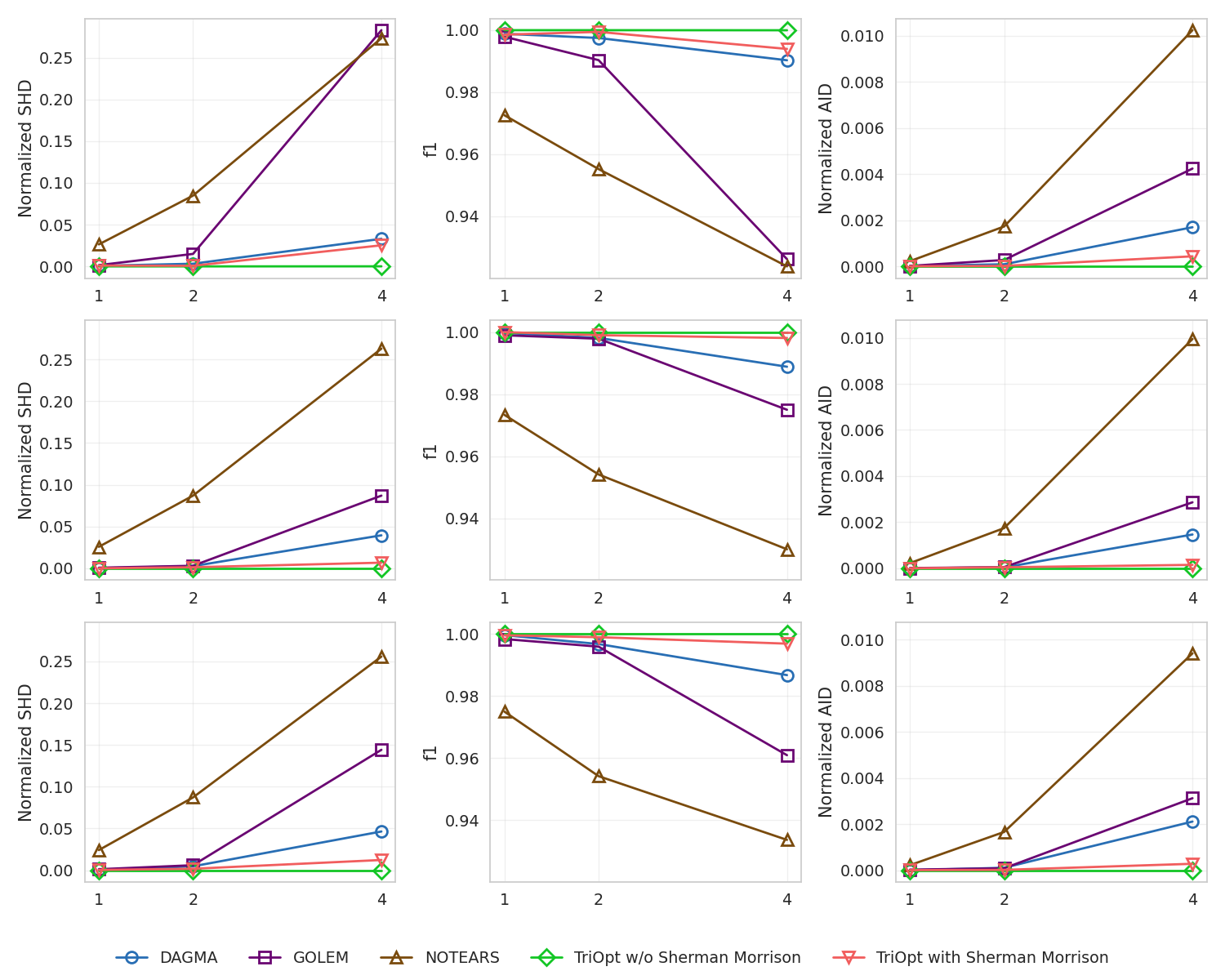}
    \caption{Robustness of the models with increase in degree. Both versions of TriOpt keep up the best performance as the graphs gets denser, while GOLEM, DAGMA and NOTEARS show significant decline. (SHD, AID Lower Better, F1 Higher Better)}
    \label{vary-sparsity}
\end{figure}

\clearpage

\subsection{Ablation study: Advantage of Upper Triangular Optimization}
\label{appendix-advantage-ut}

\begin{table}[H]
\centering

\resizebox{\textwidth}{!}{
\begin{tabular}{cccccccc}
\toprule
\multirow{2}{*}{Graph} & \multirow{2}{*}{Degree} & \multirow{2}{*}{Noise} & \multirow{2}{*}{$d$} &
\multicolumn{2}{c}{F1 Score} & \multicolumn{2}{c}{SHD} \\
\cmidrule(lr){5-6} \cmidrule(lr){7-8}
& & & & CAM Pruning & Upper Triangular & CAM Pruning & Upper Triangular \\
\midrule

\multirow{18}{*}{ER}
& \multirow{6}{*}{1}
& exponential & 20 & $0.96 \pm 0.07$ & $1.00 \pm 0.00$ & $0.03 \pm 0.04$ & $0.00 \pm 0.00$ \\
& &  & 50 & $0.869 \pm 0.04$ & $1.00 \pm 0.00$ & $0.148 \pm 0.04$ & $0.00 \pm 0.00$ \\
& & gaussian\_ev & 20 & $0.994 \pm 0.01$ & $1.00 \pm 0.00$ & $0.01 \pm 0.02$ & $0.00 \pm 0.00$ \\
& &  & 50 & $0.981 \pm 0.02$ & $1.00 \pm 0.00$ & $0.02 \pm 0.02$ & $0.00 \pm 0.00$ \\
& & gumbel & 20 & $0.977 \pm 0.02$ & $1.00 \pm 0.00$ & $0.03 \pm 0.03$ & $0.00 \pm 0.00$ \\
& &  & 50 & $0.911 \pm 0.04$ & $0.996 \pm 0.01$ & $0.096 \pm 0.04$ & $0.004 \pm 0.01$ \\

& \multirow{6}{*}{2}
& exponential & 20 & $0.987 \pm 0.02$ & $1.00 \pm 0.00$ & $0.03 \pm 0.04$ & $0.00 \pm 0.00$ \\
& &  & 50 & $0.943 \pm 0.02$ & $1.00 \pm 0.00$ & $0.108 \pm 0.03$ & $0.00 \pm 0.00$ \\
& & gaussian\_ev & 20 & $0.996 \pm 0.01$ & $0.996 \pm 0.01$ & $0.01 \pm 0.02$ & $0.01 \pm 0.02$ \\
& &  & 50 & $0.983 \pm 0.01$ & $1.00 \pm 0.00$ & $0.032 \pm 0.01$ & $0.00 \pm 0.00$ \\
& & gumbel & 20 & $0.985 \pm 0.02$ & $1.00 \pm 0.00$ & $0.04 \pm 0.07$ & $0.00 \pm 0.00$ \\
& &  & 50 & $0.948 \pm 0.02$ & $0.998 \pm 0.00$ & $0.10 \pm 0.03$ & $0.004 \pm 0.01$ \\

& \multirow{6}{*}{4}
& exponential & 20 & $0.991 \pm 0.01$ & $0.998 \pm 0.00$ & $0.04 \pm 0.02$ & $0.01 \pm 0.02$ \\
& &  & 50 & $0.977 \pm 0.02$ & $0.998 \pm 0.00$ & $0.096 \pm 0.07$ & $0.008 \pm 0.02$ \\
& & gaussian\_ev & 20 & $0.997 \pm 0.01$ & $1.00 \pm 0.00$ & $0.01 \pm 0.02$ & $0.00 \pm 0.00$ \\
& &  & 50 & $0.987 \pm 0.01$ & $0.999 \pm 0.00$ & $0.052 \pm 0.04$ & $0.004 \pm 0.01$ \\
& & gumbel & 20 & $0.995 \pm 0.01$ & $1.00 \pm 0.00$ & $0.02 \pm 0.04$ & $0.00 \pm 0.00$ \\
& &  & 50 & $0.986 \pm 0.01$ & $0.996 \pm 0.01$ & $0.056 \pm 0.03$ & $0.016 \pm 0.03$ \\

\midrule
\multirow{6}{*}{SF}
& \multirow{6}{*}{4}
& exponential & 20 & $0.987 \pm 0.01$ & $1.00 \pm 0.00$ & $0.05 \pm 0.05$ & $0.00 \pm 0.00$ \\
& &  & 50 & $0.962 \pm 0.01$ & $1.00 \pm 0.00$ & $0.152 \pm 0.06$ & $0.00 \pm 0.00$ \\
& & gaussian\_ev & 20 & $0.997 \pm 0.01$ & $1.00 \pm 0.00$ & $0.01 \pm 0.02$ & $0.00 \pm 0.00$ \\
& &  & 50 & $0.99 \pm 0.01$ & $1.00 \pm 0.00$ & $0.04 \pm 0.02$ & $0.00 \pm 0.00$ \\
& & gumbel & 20 & $1.00 \pm 0.00$ & $1.00 \pm 0.00$ & $0.00 \pm 0.00$ & $0.00 \pm 0.00$ \\
& &  & 50 & $0.983 \pm 0.02$ & $1.00 \pm 0.00$ & $0.068 \pm 0.06$ & $0.00 \pm 0.00$ \\

\bottomrule
\end{tabular}
}
\caption{Comparison of  SHD (Lower better) and F1 score (Higher Better) for CAM Pruning vs TriOpt's Optimization phase across different graph types, degrees, and noise types - under Ground Truth Ordering. TriOpt optimization outperforms CAM pruning in terms of both SHD and F1 score across all scenarios}
\label{ablation-ut}
\end{table}

\begin{table}[H]
\centering
\begin{tabular}{clc}
\toprule
Number of Nodes & Model & Time (seconds) \\
\midrule
\multirow{2}{*}{20}
 & CAM Pruning & $22.88 \pm 4.46$ \\
 & Upper Triangular Optimization & $\mathbf{0.88 \pm 0.26}$ \\
\midrule
\multirow{2}{*}{50}
 & CAM Pruning & $824.93 \pm 77.34$ \\
 & Upper Triangular Optimization & $\mathbf{1.24 \pm 0.61}$ \\
\bottomrule
\end{tabular}
\caption{CAM Pruning vs. TriOpt Optimization Time Comparison on Ground Truth Ordering for $\{20,50\}$ nodes. TriOpt is drastically faster in higher dimensions.}
\label{ablation-ut-runtime}
\end{table}

\subsection{Ablation study: Advantage of Sherman Morrison Downdate}
\label{ablation-sherman-5000}
This section shows the result of ablation study in section 5.4, based on three metrics - Order Divergence, Peak Memory Usage and Ordering time for 2000 and  5000 samples varying the nodes $\in \{100,200,500,1000\}$.
The results show Sherman Morrison downdate becomes increasingly useful in reducing the runtime of ordering phase as the number of nodes grow from 100 to 1000. However, there was a slight decline in order divergence performance and memory usage.

\begin{figure}[H]
    \centering
    \includegraphics[width=0.85\linewidth]{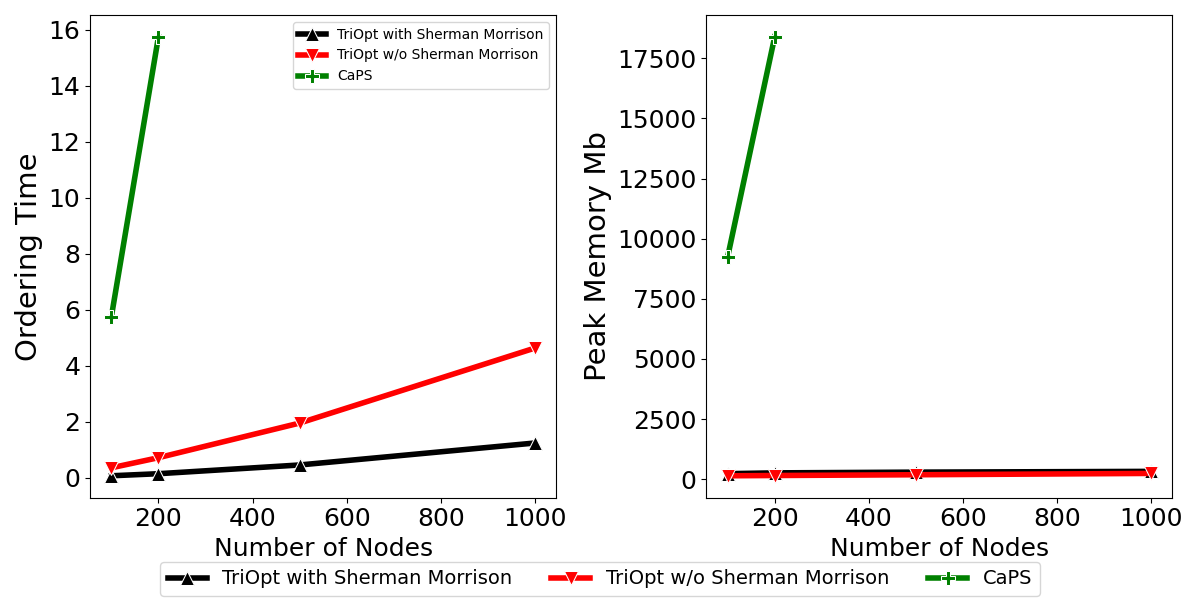}
    \caption{Comparison of Ordering time and Peak Memory usage by 2 variants of TriOpt and CaPS: TriOpt is drastically more scalable than CaPS ordering. CaPS was unable to run in dimension$>200$ due to memory bottleneck, while TriOpt completed the ordering task with order of magnitude lower memory cost}
    \label{ablation-sherman-runtime-2000}
\end{figure}

\begin{figure}[H]
    \centering
    \includegraphics[width=0.85\linewidth]{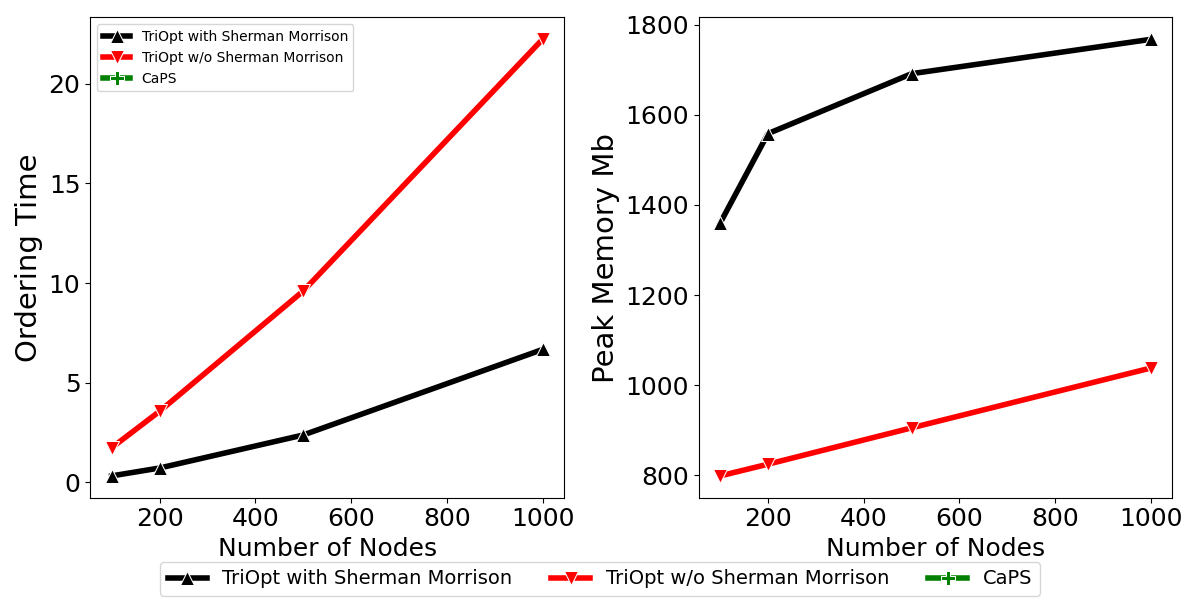}
    \caption{Comparison of Ordering time and Peak Memory usage by 2 variants of TriOpt on 5000 samples. Sherman Morrison downdate becomes more relevant in higher dimensions to reduce the runtime.}
    \label{ablation-sherman-runtime-5000}
\end{figure}

\begin{figure}[H]
    \centering
    \includegraphics[width=0.8\textwidth]{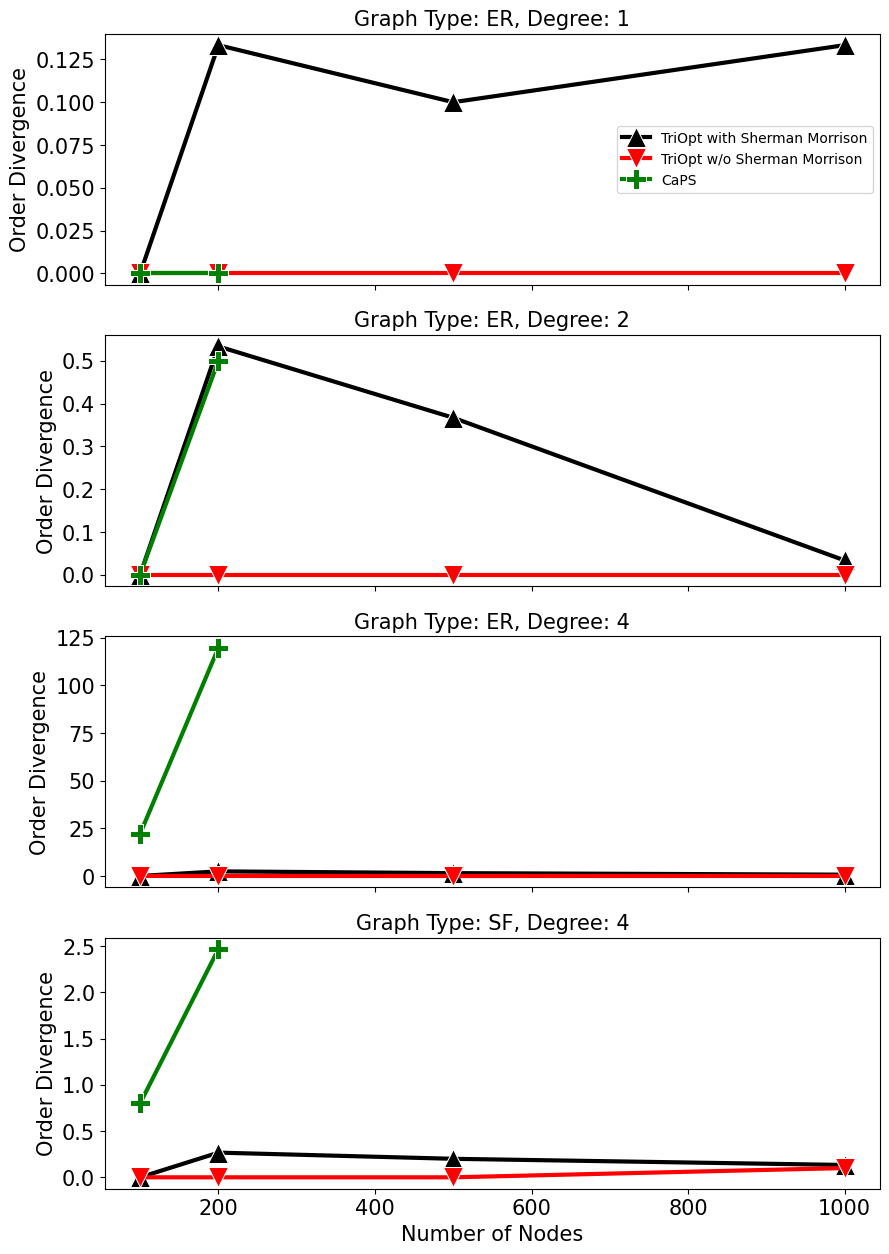}
    \caption{Comparison of Order Divergence by 2 variants of TriOpt and CaPS on 2000 samples. TriOpt without Sherman Morrison is the most accurate with order divergence always 0. CaPS ordering on the contrary gets much worse on dense graphs (ER4 and SF4)}
    \label{ablation-sherman-odivergence-2000}
\end{figure}

\begin{figure}[H]
    \centering
    \includegraphics[width=0.8\textwidth]{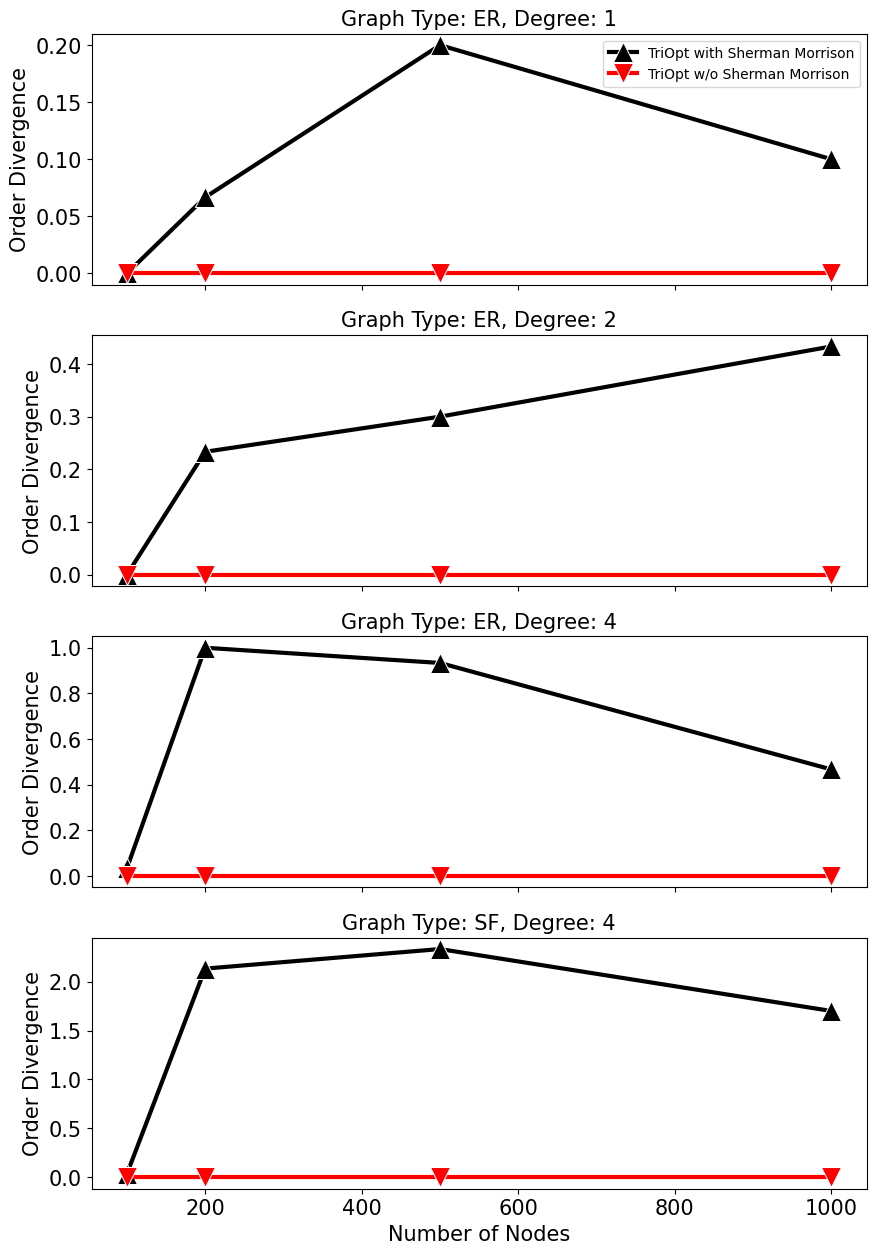}
    \caption{Comparison of Order Divergence by 2 variants of TriOpt on 5000 samples. There is negligible deterioration in performance with Sherman Morrison Downdate}
    \label{ablation-sherman-odivergence-5000}
    
\end{figure}

\subsection{Ablation Study: CaPS with Linear Kernel}
\begin{table}[H]
\centering
\scriptsize

\begin{tabularx}{\textwidth}{l *{12}{>{\centering\arraybackslash}X}}
\toprule
\textbf{graph\_degree} 
& \multicolumn{3}{c}{\textbf{ER1}} 
& \multicolumn{3}{c}{\textbf{ER2}} 
& \multicolumn{3}{c}{\textbf{ER4}} 
& \multicolumn{3}{c}{\textbf{SF4}} \\
\cmidrule(lr){2-4} \cmidrule(lr){5-7} \cmidrule(lr){8-10} \cmidrule(lr){11-13}

\textbf{model} 
& Norm.\ SHD & F1 & Runtime (s) 
& Norm.\ SHD & F1 & Runtime (s) 
& Norm.\ SHD & F1 & Runtime (s) 
& Norm.\ SHD & F1 & Runtime (s) \\
\midrule

TriOpt 
& 0.000 $\pm$ 0.000 & 1.000 $\pm$ 0.000 & 0.77 $\pm$ 0.32 
& 0.000 $\pm$ 0.000 & 1.000 $\pm$ 0.000 & 0.95 $\pm$ 0.14 
& 0.000 $\pm$ 0.000 & 1.000 $\pm$ 0.000 & 2.34 $\pm$ 0.90 
& 0.000 $\pm$ 0.000 & 1.000 $\pm$ 0.000 & 1.77 $\pm$ 0.63 \\

\addlinespace

linearCaPS 
& 0.001 $\pm$ 0.001 & 0.942 $\pm$ 0.045 & 553.4 $\pm$ 412.8 
& 0.001 $\pm$ 0.001 & 0.968 $\pm$ 0.024 & 262.1 $\pm$ 76.0 
& 0.001 $\pm$ 0.001 & 0.987 $\pm$ 0.010 & 272.1 $\pm$ 31.6 
& 0.001 $\pm$ 0.001 & 0.987 $\pm$ 0.010 & 251.9 $\pm$ 21.8 \\

\bottomrule
\end{tabularx}
\caption{CaPS with linear kernel vs TriOpt on 50 nodes. TriOpt is order of magnitude faster while acheiving better accuracy.}
\label{tab:linearCaPS}
\end{table}

\subsection{Results on Semi-synthetic and Real world data}

\begin{table}[H]
\centering
\scriptsize
\setlength{\tabcolsep}{3pt}

\begin{tabular}{lcccccc}
\toprule
& \multicolumn{2}{c}{\textbf{ARTH150}}
& \multicolumn{2}{c}{\textbf{Syntren}}
& \multicolumn{2}{c}{\textbf{Sachs}} \\
\cmidrule(lr){2-3} \cmidrule(lr){4-5} \cmidrule(lr){6-7}
\textbf{Model}
& \textbf{AID} $\downarrow$
& \textbf{F1} $\uparrow$
& \textbf{AID} $\downarrow$
& \textbf{F1} $\uparrow$
& \textbf{AID} $\downarrow$
& \textbf{F1} $\uparrow$ \\
\midrule

DAGMA
& $402.0 \pm 0.0$
& $0.447 \pm 0.00$
& $57.7 \pm 28.61$
& $0.17 \pm 0.09$
& $\mathbf{24.0 \pm 0.0}$
& $0.250 \pm 0.000$ \\

GOLEM
& $531.0 \pm 0.0$
& $0.125 \pm 0.00$
& $65.0 \pm 35.36$
& $0.16 \pm 0.07$
& $43.4 \pm 1.3$
& $0.264 \pm 0.064$ \\

NOTEARS
& $416.6 \pm 1.67$
& $0.388 \pm 0.01$
& $53.0 \pm 8.49$
& $0.20 \pm 0.15$
& $\mathbf{24.0 \pm 0.0}$
& $\mathbf{0.412 \pm 0.000}$ \\

TriOpt w/o SM
& $\mathbf{329.0 \pm 0.0}$
& $\mathbf{0.54 \pm 0.00}$
& $\mathbf{43.4 \pm 15.97}$
& $\mathbf{0.23 \pm 0.12}$
& $26.6 \pm 0.5$
& $0.283 \pm 0.046$ \\

TriOpt with SM
& $\mathbf{329.0 \pm 0.0}$
& $\mathbf{0.54 \pm 0.00}$
& $44.4 \pm 14.04$
& $0.22 \pm 0.11$
& $27.0 \pm 0.0$
& $0.250 \pm 0.000$ \\

\bottomrule
\end{tabular}
\caption{Performance comparison on ARTH150, Syntren, and Sachs datasets using AID and F1}
\label{tab:merged_aid_f1}
\vspace{-10pt}
\end{table}

\subsection{Results for All 3 Noise Types}
\label{full_result}
This section shows the results for 3 different noise types (Gaussian EV, Gumbel and Exponential) for 5000 and 2000 samples. TriOpt consistently outperforms baselines in terms of runtime. In terms of SHD, F1 and AID, TriOpt has the best performance in all noise types of ER4 and SF4 graphs across all node numbers. For ER1 and ER2 graphs, TriOpt achieved the best performance in most cases; in the few exceptions, the difference from the top-performing method and TriOpt was negligible.

\begin{figure*}[htbp]
    \centering
    \begin{subfigure}[t]{0.49\textwidth}
        \centering
        \includegraphics[width=\linewidth]{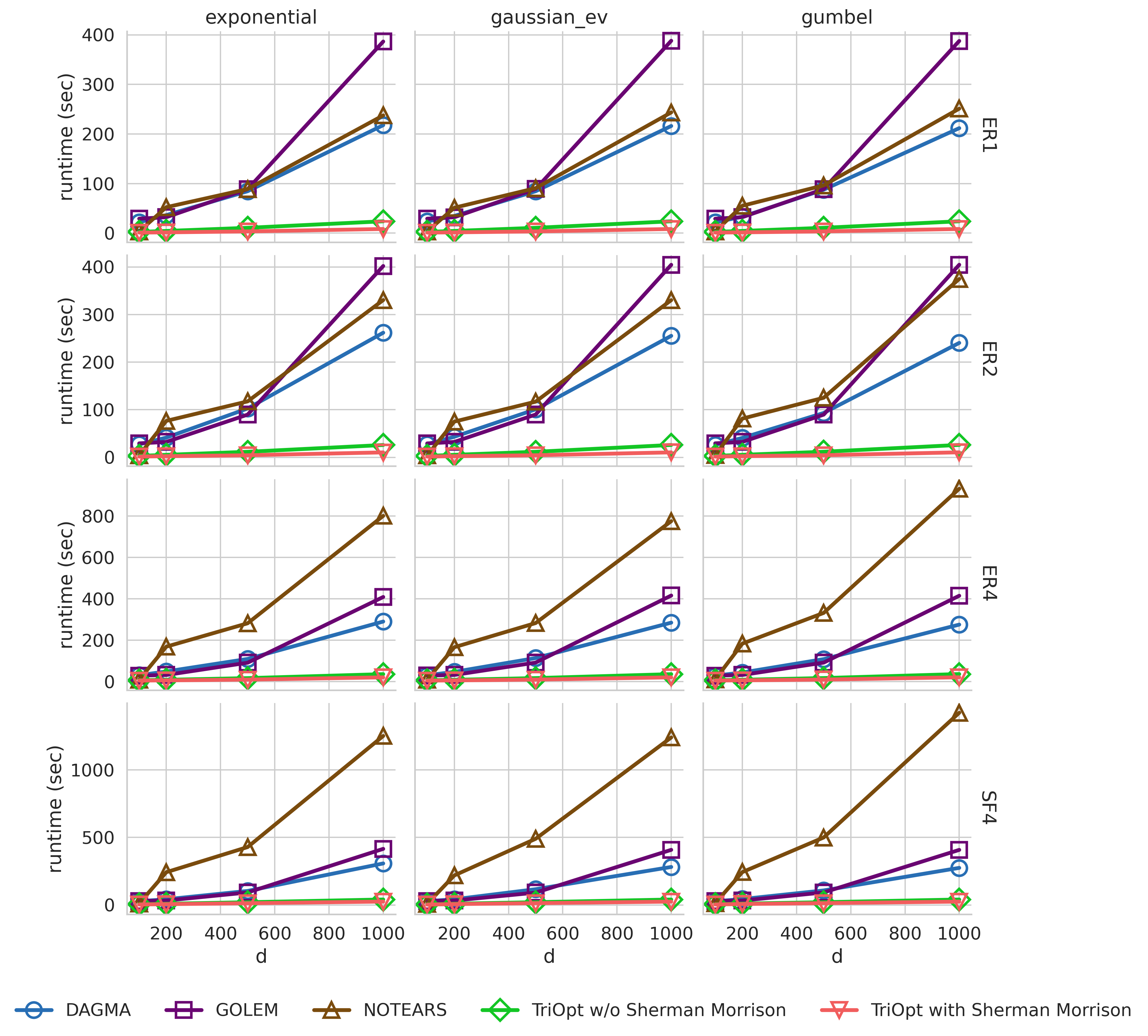}
        \caption{\scriptsize Runtime (Lower Better)}
    \end{subfigure}
    \hfill
    \begin{subfigure}[t]{0.49\textwidth}
        \centering
        \includegraphics[width=\linewidth]{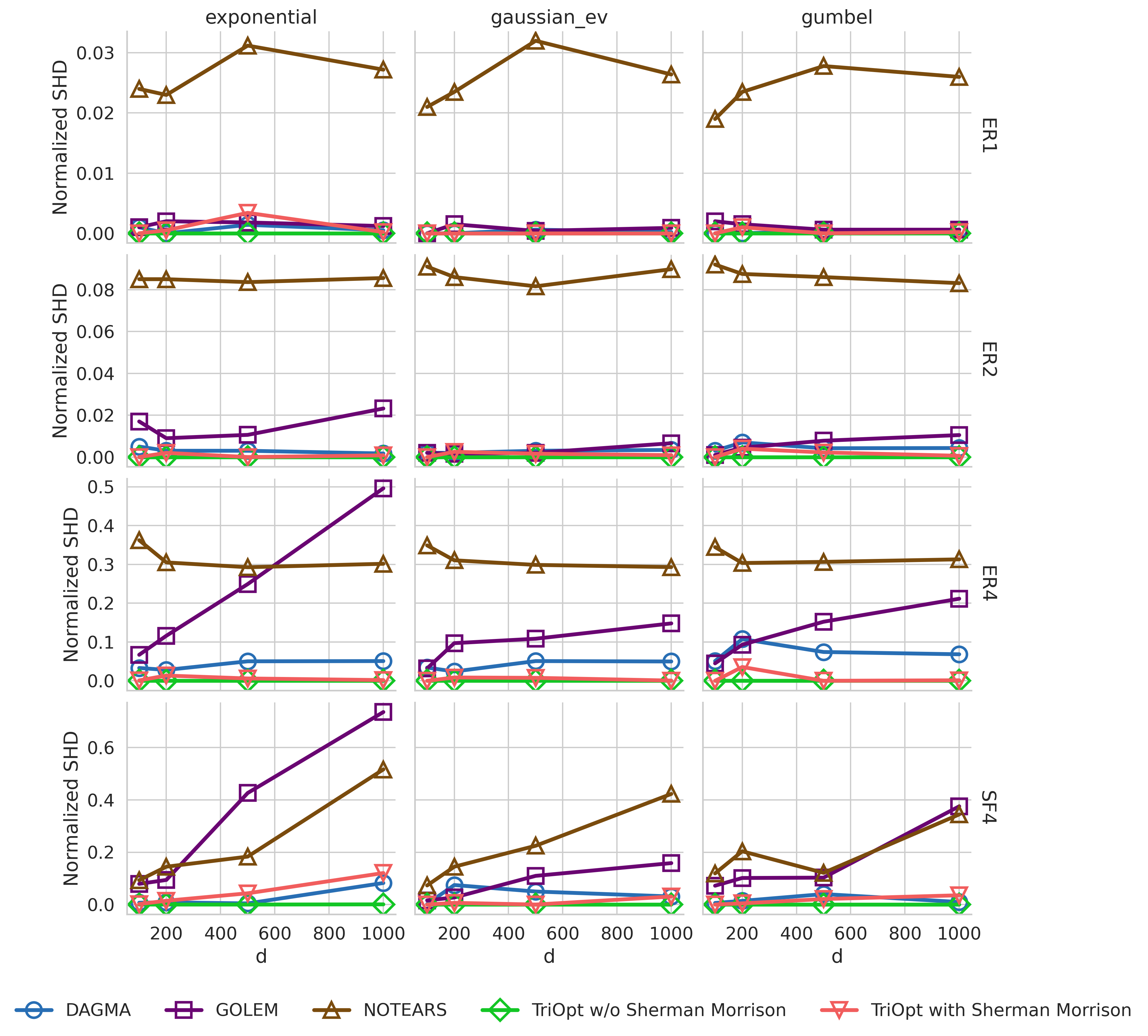}
        \caption{\scriptsize SHD (Lower Better)}
    \end{subfigure}

    \vspace{0.6em}

    \begin{subfigure}[t]{0.49\textwidth}
        \centering
        \includegraphics[width=\linewidth]{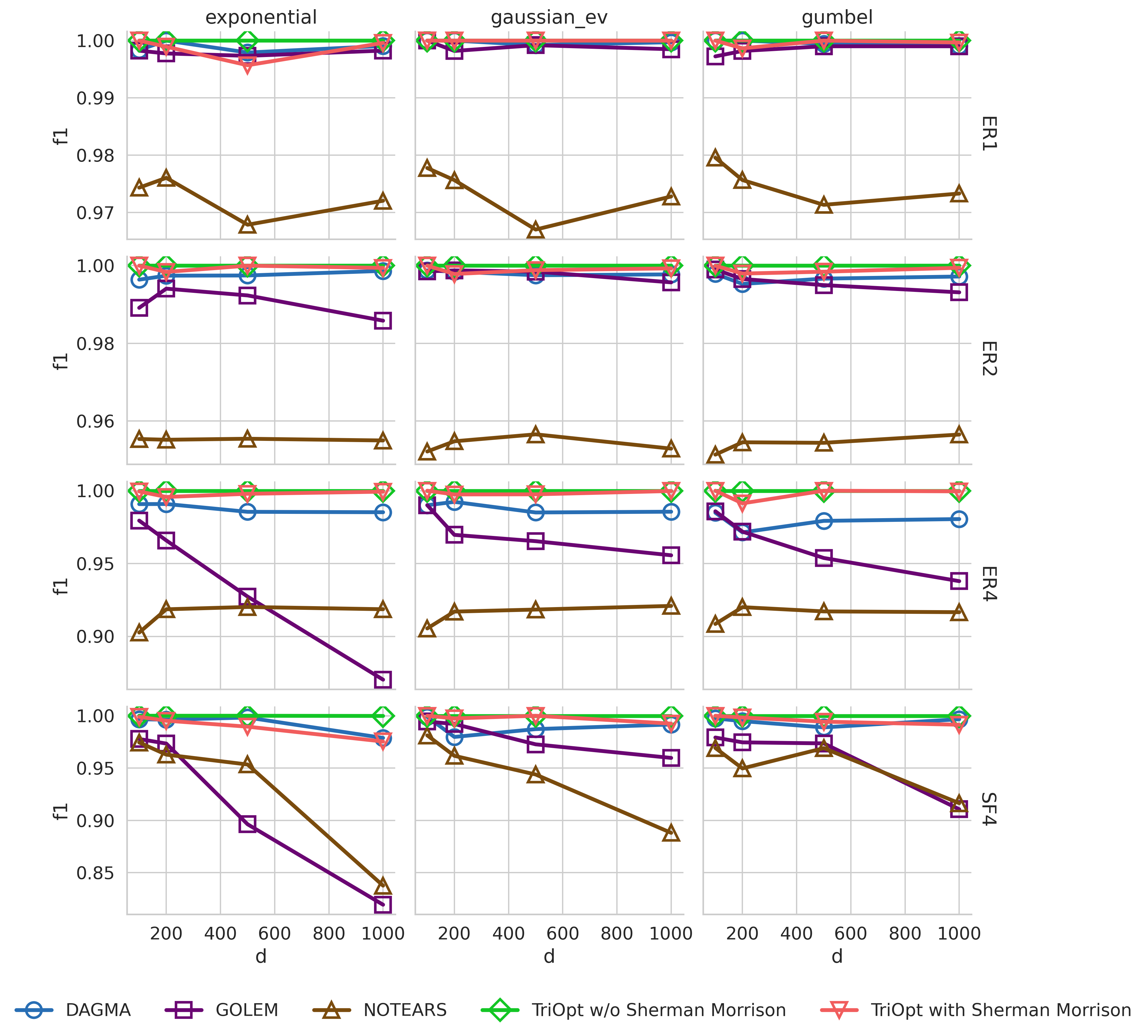}
        \caption{\scriptsize F1 score (Higher Better)}
    \end{subfigure}
    \hfill
    \begin{subfigure}[t]{0.49\textwidth}
        \centering
        \includegraphics[width=\linewidth]{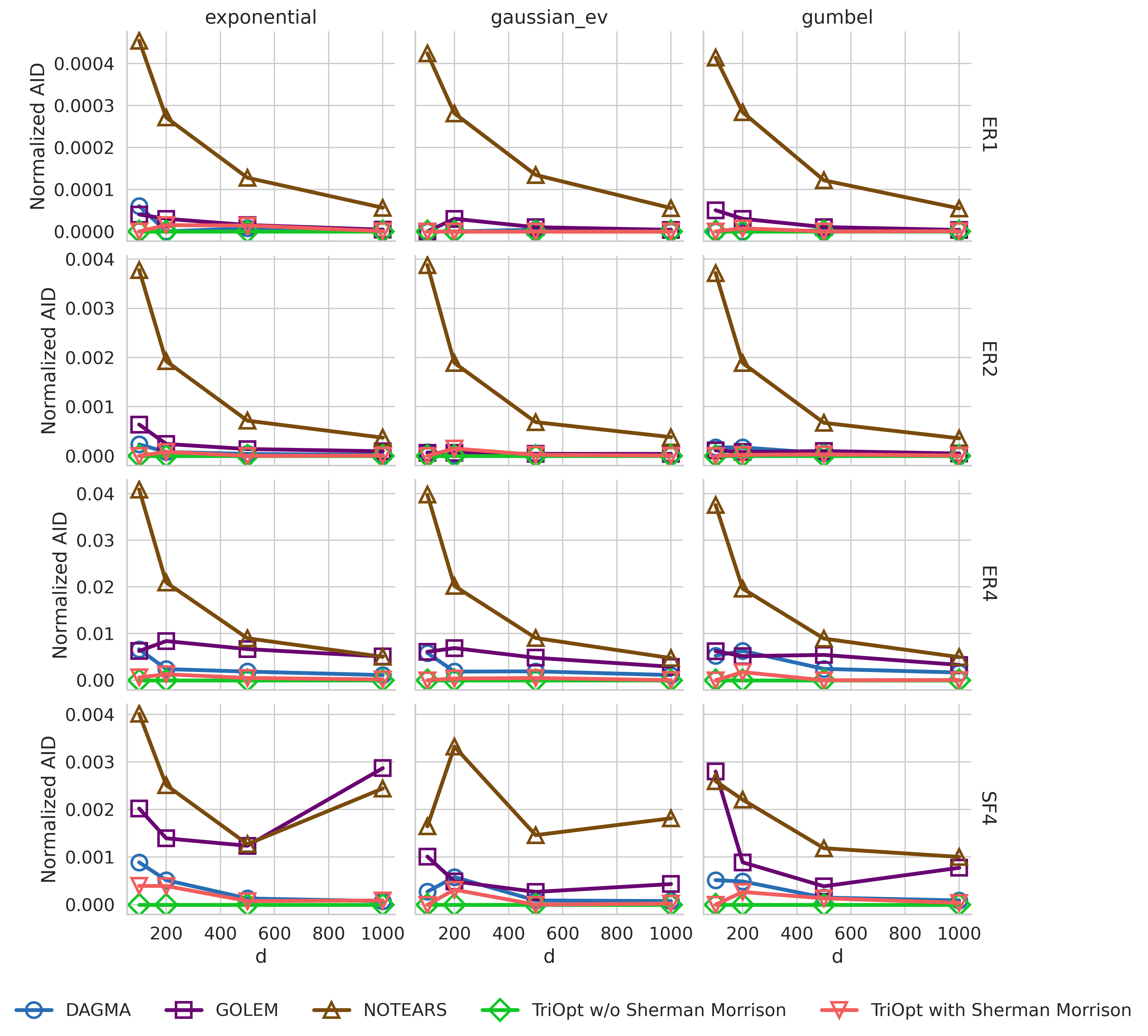}
        \caption{\scriptsize AID (Lower Better)}
    \end{subfigure}

    \caption{Performance comparison on 5000 samples.
    Results are shown for (a) runtime, (b) structural Hamming distance (SHD),
    (c) F1 score, and (d) AID, evaluated across 2 versions of TriOpt and 2 baselines GOLEM and DAGMA.}
    \label{fig:performance_5000}
\end{figure*}

\begin{figure*}[htbp]
    \centering
    \begin{subfigure}[t]{0.49\textwidth}
        \centering
        \includegraphics[width=\linewidth]{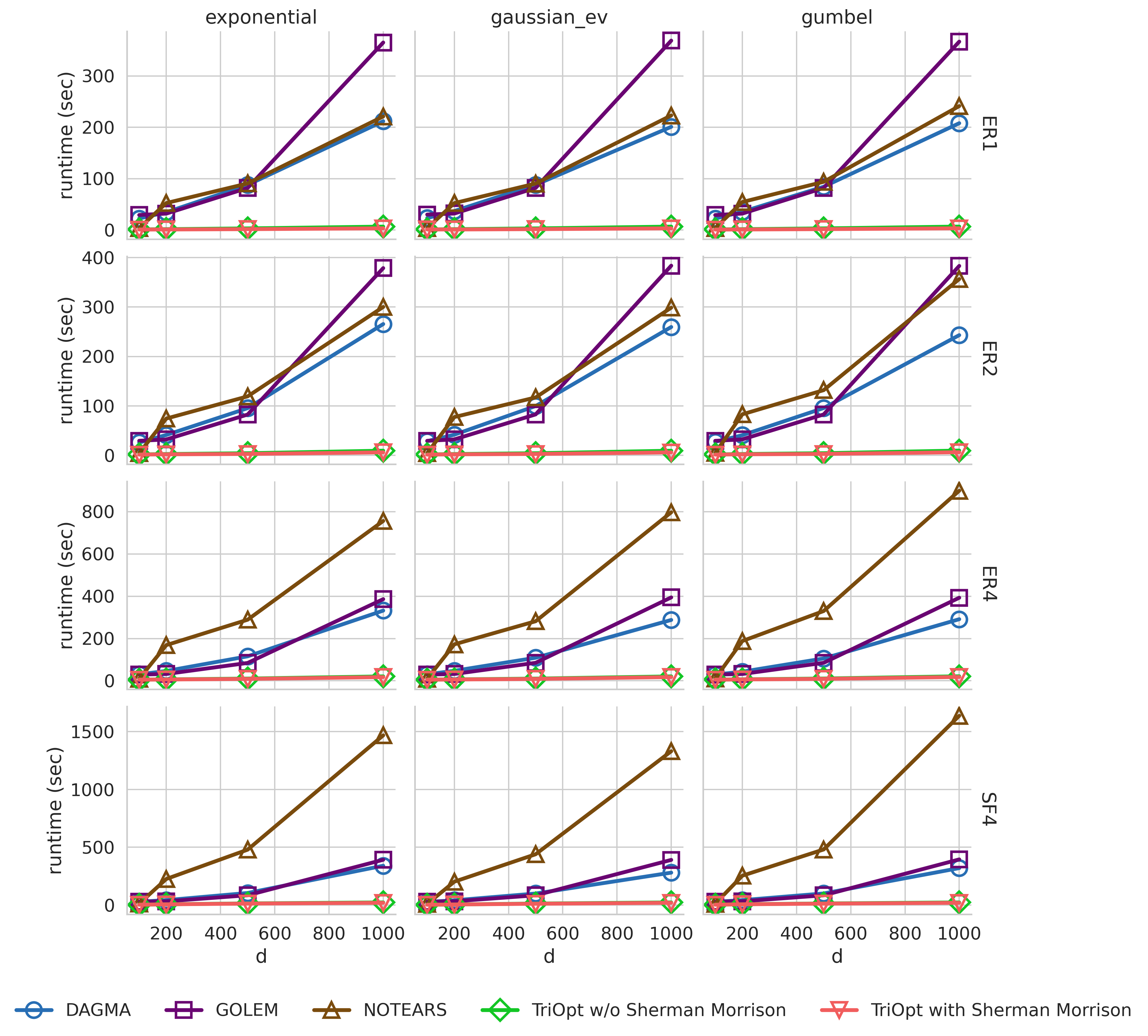}
        \caption{\scriptsize Runtime (Lower Better)}
    \end{subfigure}
    \hfill
    \begin{subfigure}[t]{0.49\textwidth}
        \centering
        \includegraphics[width=\linewidth]{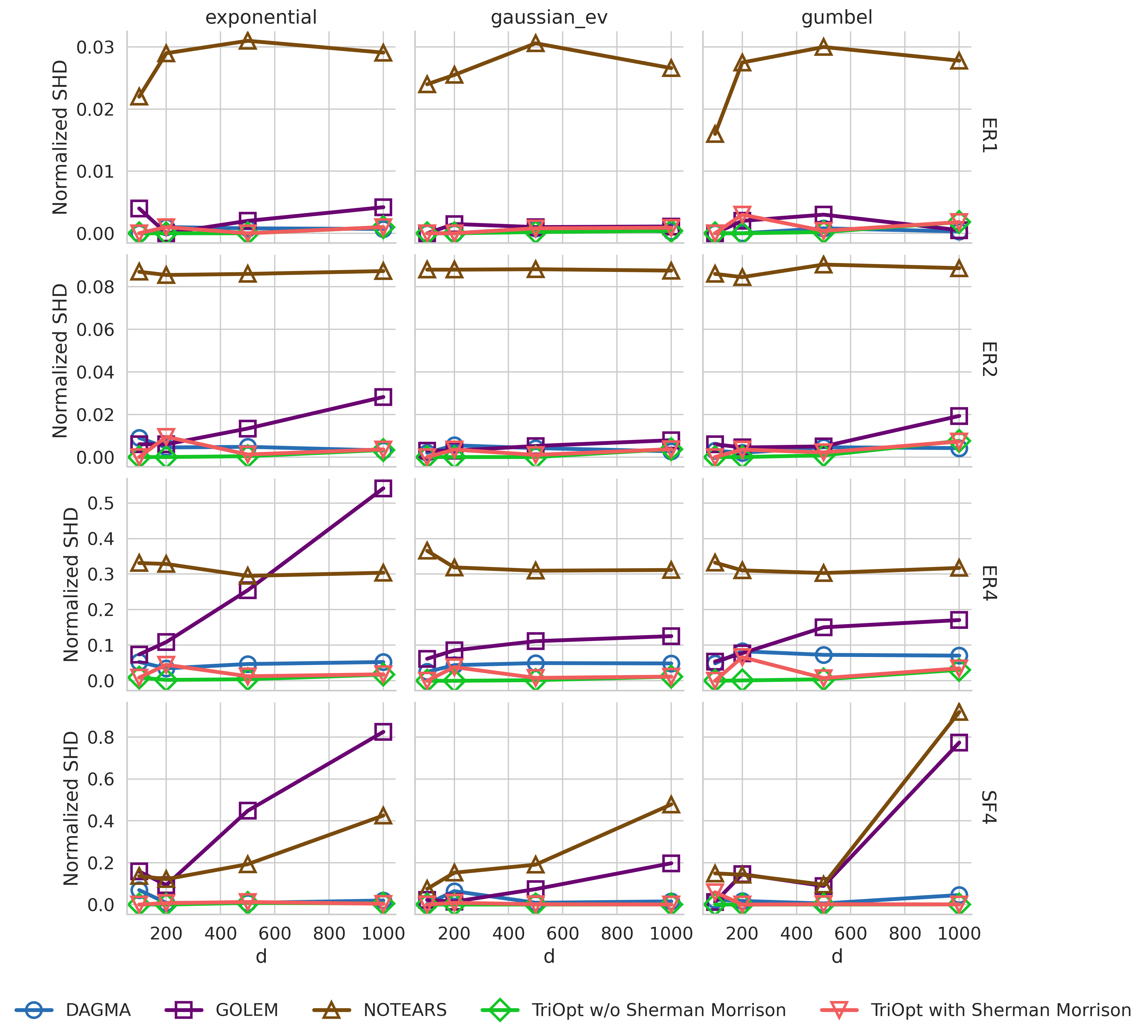}
        \caption{\scriptsize SHD (Lower Better)}
    \end{subfigure}

    \vspace{0.6em}

    \begin{subfigure}[t]{0.49\textwidth}
        \centering
        \includegraphics[width=\linewidth]{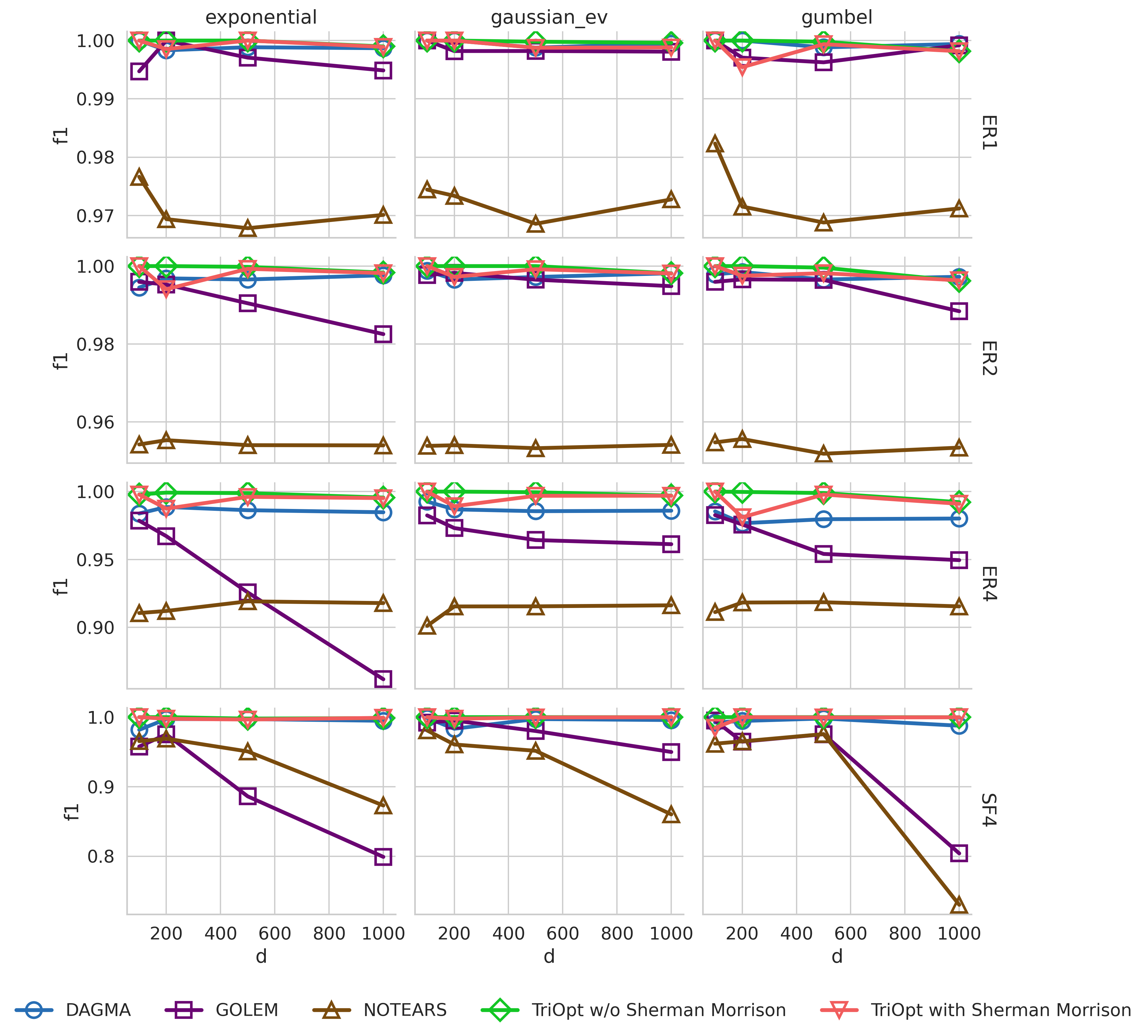}
        \caption{\scriptsize F1 score (Higher Better)}
    \end{subfigure}
    \hfill
    \begin{subfigure}[t]{0.49\textwidth}
        \centering
        \includegraphics[width=\linewidth]{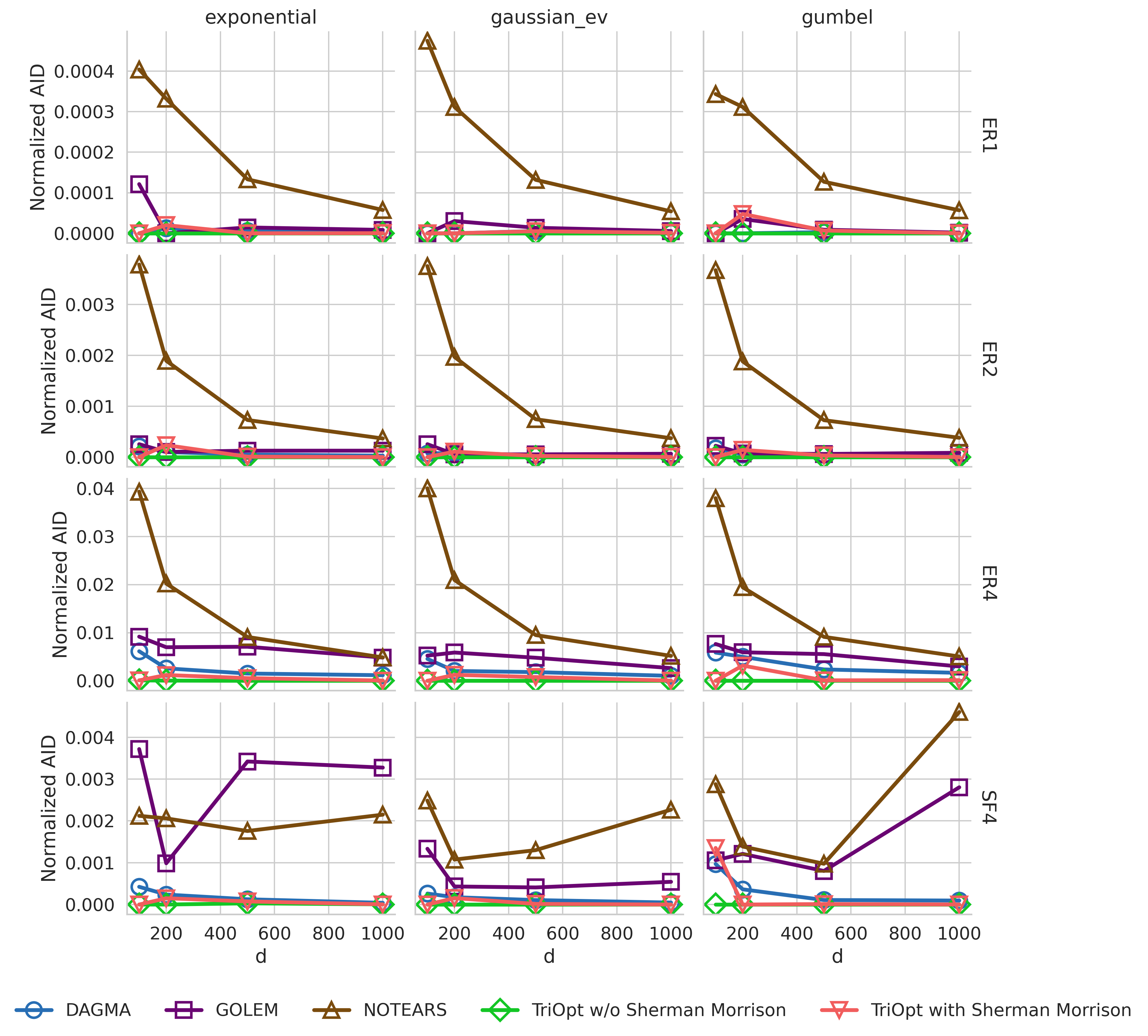}
        \caption{\scriptsize AID (Lower Better)}
    \end{subfigure}

    \caption{Performance comparison on 2000 samples.
    Results are shown for (a) runtime, (b) structural Hamming distance (SHD),
    (c) F1 score, and (d) AID, evaluated across 2 versions of TriOpt and 2 baselines GOLEM and DAGMA.}
    \label{fig:performance_2000}
\end{figure*}


\end{document}